\documentclass[10pt,twocolumn,letterpaper]{article}

\usepackage{cvpr}              

\usepackage[utf8]{inputenc} 
\usepackage[T1]{fontenc}    
\usepackage{url}            
\usepackage{booktabs}       
\usepackage{amsfonts}       
\usepackage{nicefrac}       
\usepackage{microtype}      
\usepackage{xcolor,colortbl}         

\usepackage{comment}

\usepackage{times}
\usepackage{epsfig}
\usepackage{graphicx}
\usepackage{tikz}
\usepackage{tikz-3dplot}
\usetikzlibrary{arrows}
\usetikzlibrary{decorations.markings}
\usepackage{makecell}%

\usepackage[linesnumbered,lined,boxed,commentsnumbered,ruled,vlined]{algorithm2e}

\SetCommentSty{mycommfont}

\usepackage{booktabs}

\usepackage{amsfonts}
\usepackage{bm}
\usepackage{amsmath,amsthm,amssymb,rotating, mathrsfs}

\usepackage{multirow}

\usepackage{footmisc}

\usepackage{mathtools}
\usepackage{cancel} 

\usepackage{xfrac} 
\usepackage{xparse} 

\DeclareMathAlphabet{\mathpzc}{OT1}{pzc}{m}{it}
\usepackage{enumitem}

\newcommand{\cL}{\mathcal{L}}

\newcommand{\bbR}{\mathbb{R}}

\NewDocumentCommand{\norm}{mG{2}}{\big\|#1\big\|_{#2}}

\DeclareMathOperator{\diag}{diag}

\newcommand{\argmin}{\mathop{\rm argmin}}

\NewDocumentCommand{\seqp}{mG{n}}{{#1}_1-\cdots+ {#1}_{#2}}
\NewDocumentCommand{\seqm}{mG{n}}{{#1}_1-\cdots- {#1}_{#2}}

\newcommand{\myparagraph}[1]{\noindent\textbf{#1.}}

\newtheorem{lemma}{Lemma}

\theoremstyle{definition}

\theoremstyle{remark}
\newtheorem{remark}{Remark}

\definecolor{mygray}{rgb}{211,211,211}
\definecolor{cvprblue}{rgb}{0.21,0.49,0.74}
\definecolor{mypurple}{rgb}{200,200,235}
\usepackage[pagebackref,breaklinks,colorlinks,citecolor=cvprblue,linkcolor=cvprblue]{hyperref}
\usepackage[capitalise]{cleveref}
\usepackage{crossreftools}
\usepackage{booktabs} 
\usepackage{pifont} 
\newcommand{\cmark}{\ding{51}} 
\newcommand{\xmark}{\ding{55}} 
\crefformat{equation}{(#2#1#3)}

\def\paperID{22610} 
\def\confName{CVPR}
\def\confYear{2026}
\def\ourname{Gated KalmaNet\xspace}
\def\ourshortname{GKA\xspace}

\title{Gated KalmaNet: A Fading Memory Layer Through Test-Time Ridge Regression }

\author{Liangzu Peng$^{1,}$\thanks{Work done during an internship at AWS Agentic AI.} \quad Aditya Chattopadhyay$^{2,}$\thanks{Correspondence to achatto@amazon.com} \quad Luca Zancato$^{2}$ \quad Elvis Nunez$^{2}$ \quad Wei Xia$^{2}$ \quad Stefano Soatto$^{2}$\\
University of Pennsylvania$^1$ \quad AWS Agentic AI$^2$\\
{\tt\small lpenn@seas.upenn.edu }\\
{\tt\small \{achatto,zancato,elvisnun,wxia,soattos\}@amazon.com}}

\begin{document}
\maketitle


\begin{abstract}
 Linear State-Space Models (SSMs) offer an efficient alternative to softmax Attention with constant memory and linear compute, but their lossy, fading summary of the past hurts recall-oriented tasks. We propose \ourname (\ourshortname, pronounced "gee-ka"), a layer that accounts for the full past while retaining SSM-style efficiency. We ground our approach in the Kalman Filter (KF), and show that several existing SSM layers (DeltaNet, Gated DeltaNet, Kimi Delta Attention) are approximations to the KF recurrence under an identity error covariance assumption, which ignores how past keys and values should optimally influence state updates. In contrast, \ourshortname maintains the full error covariance and computes the exact Kalman gain. Under a steady-state assumption that enables parallelization, this reduces to an online ridge regression with constant memory and linear compute. The standard KF equations are numerically unstable in low-precision settings (e.g., bfloat16) and hard to parallelize on GPUs. We address this with (1) adaptive regularization via input-dependent gating to control the ridge regression's condition number, and (2) Chebyshev Iteration, which we show is more stable than conventional iterative solvers in low precision. We further develop hardware-aware chunk-wise kernels for efficient training. Empirically, \ourshortname outperforms existing SSM layers (e.g., Mamba2, Gated DeltaNet) on short-context tasks and achieves more than 10\% relative improvement on long-context RAG and LongQA up to 128k tokens. We further show \ourshortname outperforms Mamba when extended to ImageNet classification. Our code, including Triton kernels for training and inference (vLLM),\footnote{\url{https://github.com/awslabs/hybrid-model-factory}; \ourname is implemented as the \texttt{GKA} layer.} along with a model zoo of \ourshortname-based Hybrid models at 8B and 32B scale on HuggingFace,\footnote{\url{https://huggingface.co/collections/amazon/primed-hybrid-models-collection}} is released under Apache 2.0.

\end{abstract}

\section{Introduction}\label{section:intro}
Large Language Models (LLMs) powered by (softmax) Attention mechanisms \cite{Vaswani-NeurIPS2017} have revolutionized sequence modeling through their ability to form rich associations within their context window. However, a fundamental challenge that LLMs face is that their time complexity scales quadratically and storage grows linearly with their input length.

Recent years have seen intense efforts to develop Attention alternatives. Among them, memory layers based on linear State-Space models (SSMs) have grown popular for their linear-time computation and constant storage cost in the sequence length \cite{Dao-ICML2024-mamba2, Yang-ICML2024-gla}. These SSMs find inspirations from classic techniques in adaptive signal processing, and integrating them into modern SSMs leads to principled layer design and enhanced performance \cite{Liu-ICLR2025,Yang-NeurIPS2024, Zancato-NeurIPS2024-BMOJO}. However, pure SSM models still underperform Attention in many settings, especially on long-context tasks. This gap is a consequence of their different memory mechanisms: SSMs have a \textit{fading} fixed dimensional \textit{lossy} state of the past, while Attention has an \textit{eidetic} ever increasing 
KV-cache state \cite{Zancato-NeurIPS2024-BMOJO}.

To bridge this gap, we aim at designing a memory layer that enjoys the efficiency of linear SSMs while performing computation conditioned on the exact past. Towards this goal, we first draw insights from the \textit{Kalman filter} (KF) \cite{Kalman-1960}. In signal processing terms, KF computes the most recent state conditioned on all data seen thus far, and, under mild assumptions, KF is optimal in the \textit{Maximum A-Posteriori} (MAP) sense. In the LLM context, we use KF to update the state of an SSM layer and predict its output based on all past inputs. However, integrating KF into such a layer is non-trivial and faces two challenges:
\begin{itemize}
\item \textit{Parallelizable Training.} KF is an online  algorithm and needs to be parallelized to fully utilize modern hardware that is highly optimized for large-scale LLM training.
\item \textit{Numerical Stability.} KF involves matrix inversion, which can be numerically unstable in low precision arithmetic. 
\end{itemize}
In this work, we propose \textit{Gated KalmaNet} (GKA), a memory layer that incorporates KF into its design and is both numerically stable and trainable on highly parallelizable hardware. 
We start by observing that the KF recursion solves a test-time ridge regression problem. Then, to solve such a regularized problem stably, we make the following choices:
\begin{itemize}
    \item At the modeling level, we adaptively choose the regularization strength of our test-time objective function based on the Frobenius norm of the regularized data covariance. With this choice we can easily upper bound the condition number of the optimization problem. 
    \item At the algorithmic level, we note that exact solvers (e.g., \texttt{torch.linalg.solve}) are hard to parallelize (in a chunk-wise manner), so we resort to the classic \textit{Chebyshev Iteration} (CH), which we show has high numerical accuracy and fast convergence compared with alternatives such as (accelerated) gradient descent and conjugate gradient. 
\end{itemize}
To make \ourshortname scalable and efficient, we implement CH with adaptive regularization in Triton in a hardware-aware, chunk-wise manner. Our technical novelty here includes deriving a chunk-wise implementation that back-propagates through the Frobenius norm, for which the difficulty is the presence of a \textit{nested} recurrence. 
Furthermore, we combine CH with a \textit{gating mechanism} that decides the regression residual weights in an input-aware and time-varying fashion, enhancing the contribution of recent inputs and smoothly fading out distant contexts. Overall, to the best of our knowledge, this is a first adoption of the CH method for training sequence modeling layers in LLMs stably at scale.

Empirically, thanks to \ourshortname's computations conditioned on the entire past, it consistently outperforms state-of-the-art linear SSMs, including Mamba2 \cite{Dao-ICML2024-mamba2} and (Gated) DeltaNet  \cite{Yang-NeurIPS2024,Yang-ICLR2025}, across Synthetic Recall (MQAR \cite{arora2023zoology}), short-context (LM-Harness \cite{eval-harness}), and long-context benchmarks (RULER \cite{hsieh2024ruler}, HELMET \cite{yen2025helmet}). On real-world long-context tasks such as Retrieval-Augmented Generation and Long Question-Answering at 128k tokens, \ourshortname improves upon SSM baselines by at least 10\%. Beyond language, GKAVision outperforms MambaVision on ImageNet classification, suggesting that \ourshortname's recall advantages transfer across modalities.

\section{Prior Work and Preliminaries}
In this section we briefly review prior work and preliminaries that will set the stage for motivating our work. For a more detailed exposition of related work see \cref{appendix:related work}. 

\myparagraph{(Softmax) Attention} At each time $t$, Attention \cite{Vaswani-NeurIPS2017} linearly projects the $t$-th input token to obtain three vectors, named  \textit{query} $q_t$, \textit{key} $k_t$, \textit{value} $v_t$ respectively. Then, it outputs a vector $y_t \in \bbR^D$ as a convex combination of all values seen so far, with coefficients $c_1,\dots,c_t$ given by inner products of the current query $q_t$ with all seen keys and a softmax mapping:
\begin{align}\label{eq:attn}
    y_t =\sum_{i=1}^t  c_i v_i,\quad \quad c_i:= \frac{ \exp(\frac{k_i^\top q_t}{\sqrt{D}}) }{\sum_{i=1}^t \exp(\frac{k_i^\top q_t}{\sqrt{D}}) } . \tag{Attn}
\end{align}
From an optimization perspective, \cref{eq:attn} can be viewed as solving the following \textit{regression}  objective\footnote{Concretely, for keys and queries of unit norm, Attention is precisely the \textit{Nadaraya-Watson estimator} \cite{Nadaraya-1964,Watson-1964} with the Gaussian kernel to approximate the conditional expectation of the value given a query; cf. \cite{Chaudhari-TIST2021,Vidal-2022}.\label{footnote:kernelreg}},
\begin{align}\label{eq: attn as regression}
    y_t = \argmin_{v}  \sum_{i = 1}^t\exp\left(\frac{k_i^\top q_t }{\sqrt{D}} \right) \cdot \| v - v_{i} \|_2^2. 
\end{align}
The success of \cref{eq:attn} is often attributed to its ability to perform verbatim retrieval of relevant context from the entire past. Here, the past refers to  the entire key-value pairs observed thus far, also known as the \textit{KV-cache}, which grows linearly with time $t$. Moreover, the computation is also linear at each time $t$, and doing so for all $t$ results in a quadratic time complexity. This high computation and storage cost of Attention makes its use prohibitive in long context scenarios.

\myparagraph{Linear State-Space Models (SSMs)} The high computation cost of \cref{eq:attn} has motivated a flurry of work developing new LLM layers, like SSMs, with linear rather than quadratic cost. 
Most SSMs maintain a \textit{state matrix} $S_t\in \bbR^{D\times D}$ and update it at each time step via a linear recursion of the form
\begin{align}\label{eq:linear-attn}
    S_t = \gamma_t \cdot S_{t-1} + \beta_{t} \cdot  v_t k_t^\top, \quad y_t = S_t q_t, \tag{Linear-SSM}
\end{align}
where $\gamma_t,\beta_t$ are typically in $[0,1]$. 
Unlike the verbatim lookup of \cref{eq:attn}, here \cref{eq:linear-attn} essentially compresses the entire KV-cache into a fixed-dimensional representation $S_t$. Subsequent computation of the output $y_t$ relies on $S_t$ and no longer on the exact past. This results in a constant cost of storage and computation at every timestep.

In many linear SSMs (e.g., RetNet \cite{Sun-arXiv2023RetNet}, Mamba2 \cite{Dao-ICML2024-mamba2}), the use of $\gamma_t$ and $\beta_t$ is often heuristic and finds inspirations from nonlinear recurrent neural networks \cite{Hochreiter-NC1997}; in that light, $\gamma_t$ and $\beta_t$ are called \textit{forgetting} and \textit{input} gates, respectively. This basic form of \cref{eq:linear-attn} has been generalized by replacing $\gamma_t$ with a diagonal matrix (GLA \cite{Yang-ICML2024-gla}, RWKV-6 \cite{Peng-CoLM2024}, Longhorn \cite{Liu-ICLR2025}) or \textit{low-rank-plus-identity} matrix (Gated DeltaNet \cite{schlag2021linear,Yang-NeurIPS2024,Yang-ICLR2025}, DeltaProduct \cite{Siems-arXiv2025}, RWKV-7 \cite{Peng-arXiv2025rwkv}). 

Similarly to that of \cref{eq:attn} the case with \textit{low-rank-plus-identity} matrices can often be justified from an optimization perspective. For example,  Gated DeltaNet  \cite{schlag2021linear,Yang-NeurIPS2024} updates the state via ($I$ is the $D\times D$ identity matrix)
\begin{align}\label{eq:GDN}
    S_t = \gamma_t \cdot S_{t-1} \left(I - \beta_t k_t k_t^\top\right) + \beta_t \cdot v_t k_t^\top, \tag{GDN} 
\end{align}
which can be viewed as applying one gradient descent step with step-size $\beta_t$ and initialization $\gamma_tS_{t-1}$ to the objective 
\begin{align}\label{eq:GDN-obj}
    \min_{S} \| S k_t - v_t \|_2^2.
\end{align}
The objectives of \cref{eq:GDN} \cref{eq:GDN-obj} and \cref{eq:attn} \cref{eq: attn as regression} are prime examples that expose a general distinction between linear SSMs and \cref{eq:attn}: The former updates its state based on a regression objective that considers only the previous \textit{lossy} state and the current time step, whereas the latter uses the entire, exact KV-cache to solve its regression objective \cref{eq: attn as regression}. 

We hypothesize this myopic view of SSM objectives results in their lower performance and limited long-context abilities. We then ask: \textit{What is an objective or, equivalently, a recursion that considers the entire past as \cref{eq:attn} while still being solvable in linear time as in \cref{eq:linear-attn}?}

\section{A Linear SSM Inspired by the Kalman Filter}
In \cref{subsection:KF-motivation} we show how the  Kalman Filter (KF) gives insights into a new linear SSM layer that takes all past time instants into account. In \cref{subsection:hurdle} we explain the numerical and efficiency challenges of building such a layer.

\subsection{Motivation from Kalman Filter}\label{subsection:KF-motivation}
KF is an established online approach that takes the exact past into account to optimally solve a \textit{weighted ridge regression} objective (e.g., see \cite[Proposition 2 \& Lemma 3]{Peng-MoCL2025}). In our context, this means that the \textit{optimal} state 
\begin{align}\label{eq:rls}
    \begin{split}
        S_{t} &= \argmin_{ S\in\bbR^{D\times D} } \lambda \cdot \| S \|_{\text{F}}^2 + \sum_{i=1}^t \eta_{i}\cdot \| S k_{i} - v_{i}\|_2^2
    \end{split} 
\end{align}
can be computed by the KF recursion
\begin{align}\label{eq:kf-update}
    S_{t} = S_{t-1} - \frac{(S_{t-1} k_t -v_t) k_t^\top \Phi_{{t}-1}}{1/\eta_{t} + k_t^\top \Phi_{{t}-1}k_t},  \tag{KF}
\end{align}
where $\eta_t$ is the weight for the $t$-th key-value pair, and $\Phi_{t-1}$ is the Hessian inverse of \cref{eq:rls} at time $t-1$  ($\Phi_{t-1}$ itself can be continually updated via the \textit{Woodbury matrix identity}). It is now clear that objective \cref{eq:rls} takes the entire KV-cache into account, similarly to \cref{eq:attn}. It is also clear that \cref{eq:kf-update} is an efficient update scheme similarly to \cref{eq:linear-attn}; indeed, \cref{eq:kf-update} is also a \textit{low-rank-plus-identity} form (cf. \cref{eq:GDN}). 

While \cref{eq:linear-attn} relies on instantaneous objectives akin to \cref{eq:GDN-obj} (cf. \cite[Table 2]{Yang-NeurIPS2024}), \cref{eq:kf-update} leverages second-order information from $\Phi_{t-1}$ to solve \cref{eq:rls} optimally. It is in this sense that we say \cref{eq:kf-update} is more expressive than other \cref{eq:linear-attn} or \cref{eq:GDN}. We now detail the differences in the objectives of \cref{eq:kf-update} and \cref{eq:attn}:

\begin{itemize}
    \item \cref{eq:kf-update} computes a parametric linear estimator that enables a constant-sized memory, while \cref{eq: attn as regression} computes a non-parametric point estimate that entails storing the full cache.
    \item In \cref{eq: attn as regression}, the weights of the same residual vary over time as the queries differ at each time, while in \cref{eq:rls} $i$-th weight $\eta_i$ is constant once observed at time $i$. The former results in quadratically many weights---thus a quadratic time complexity---and the latter linearly many.    
    \item In \cref{eq:rls}, the regularizer $\lambda \cdot \| S \|_{\text{F}}^2$ prevents overfitting our state to key-value pairs, as only a finite amount of ``information" can be stored in a constant-sized memory beyond which will result in ``fuzzy" recall. In this light, $\lambda$ can be thought of as controlling the memorization ``capacity" of the state.
\end{itemize}


\subsection{Hurdles Towards Scalable Kalman Filter SSMs}\label{subsection:hurdle}
Despite its optimality and (sequential) efficiency
the \cref{eq:kf-update} recursion lacks a hardware-aware implementation that leverages parallelism in modern Tensors Cores. Moreover, for long sequences it can lose numerical precision due to division (and due to how $\Phi_{t}$ is updated). The final hurdle is conceptual: Fixing weights $\eta_i$ and regularization $\lambda$ over time as in \cref{eq:rls} might make a layer less expressive.

We are aware of the use of \cref{eq:kf-update} or \cref{eq:rls} in neural network training three decades ago \cite{Shah-1992} or in deep continual learning recently \cite{Zeng-NMI2019,Mcdonnell-NeurIPS2023,Peng-ICLR2025}. We are also aware of the recent mentioning of \cref{eq:rls} or efforts towards solving it, which go by the name \textit{test-time optimization} \cite{Wang-arXiv2025v3,Von-arXiv2025-mesa}. However, to the best of our knowledge, none of the prior work has fully addressed the above hurdles that need to be solved to design an 
SSM layer that is trainable in parallel, numerical well-behaved, and sufficiently expressive. 
In particular, both \cite{Von-arXiv2025-mesa} and \cite{Wang-arXiv2025v3} have overlooked a basic numerical concern: The worst-case numerical error in solving   \cref{eq:rls} can be $\epsilon\cdot \kappa$ \cite{Golub-2013}, where $\kappa$ is the condition number of the Hessian in \cref{eq:rls} and $\epsilon$ the machine precision; since $\epsilon\approx0.007$ (bf16), \cref{eq:rls} has to be regularized \textit{strongly} for $\kappa$ and the worst-case error to be small, regardless of algorithmic choices to solve \cref{eq:rls}. Indeed, the regularization  enforced in \cite{Von-arXiv2025-mesa} sets $\lambda$ to be lower bounded by $0.25$, but this is not sufficient: Their $\kappa$ is as large as $500$ \cite[Fig. 13]{Von-arXiv2025-mesa}, implying a worst-case error of $3.5$ (The implementation of \cite{Von-arXiv2025-mesa} available on GitHub is numerically vulnerable; we failed to train it without NaNs in various settings.). Also, the  regression objective in \cite{Wang-arXiv2025v3} has no regularization, which makes it numerically ill-posed for low-precision training.

\section{\ourname (\ourshortname)}
\label{sec: GKA}
We propose \textit{\ourname} (\ourshortname) to address the above hurdles: We enhance numerical stability via adaptive regularization and the classic \textit{Chebyshev Iteration} (CH), increase expressivity of KF via a standard gating mechanism, and improve parallelism via a hardware-friendly implementation.

\subsection{CH with Adaptive Regularization \& Weighting}\label{subsection:CH-adaptive-reg}
\myparagraph{Motivation} As alluded earlier, solving \cref{eq:rls} via \cref{eq:kf-update} is sequential in nature, and here we consider alternatives amenable to parallelizable training. Our first step towards this is to write down a closed form solution to \cref{eq:rls} and compute the output
\begin{align*}
    y_t = S_t q_t = \left( \sum_{i=1}^t \eta_i v_ik_i^\top \right) \left(\sum_{i=1}^t \eta_i k_i k_i^\top + \lambda I \right)^{-1} q_t.
\end{align*}
With the weighted covariances $U_t:= \sum_{i=1}^t \eta_i v_ik_i^\top$ and $H_t:=\sum_{i=1}^t \eta_i k_i k_i^\top$, we note that $y_t$ can be computed via first solving $(H_t + \lambda I) x = q_t$ for $x$ and then left-multiplying $U_t$. An exact solver (e.g., \texttt{torch.linalg.solve}) can do so with high accuracy by parallelizing over the batch dimension. However, it is inefficient for two reasons. First, it takes $O(D^3)$ time for every $t$. Second, it requires explicitly forming and materializing all $H_t$'s, which entails large I/O costs. In light of this, we resort to first-order iterative methods. These methods use matrix-vector products to enable chunk-wise parallelism over batches without materializing all $H_t$'s. Furthermore, they often take $O(D^2)$ time complexity per iteration and can converge quickly in a few iterations. The iterative method we choose is the \textit{Chebyshev Iteration} (CH);  we proceed to describe its basic idea, with the justification for using CH deferred to \cref{subsection:solver-comparison}.

\myparagraph{Chebyshev Iteration (CH)} CH is an \textit{accelerated gradient descent} method (AGD) that applies \cref{eq:ch-gd} and \cref{eq:ch-momentum} to the quadratic objective $\frac{1}{2} \xi^\top H \xi - \xi^\top q$, that is to solve $H \xi = q$  (\cref{algo:Chebyshev}). Different from vanilla AGD, CH incorporates a \cref{eq:weight-update} and makes specific choices of different parameters; these choices make CH optimal with the fastest convergence among first-order methods \cite{Pedregosa-Chebyshev2021}. 

We replace the above exact solver with CH: 
\begin{align*}
    \hat{x}_t = \text{CH}(H_t+\lambda I, q_t, r),  \quad \quad 
    y_t= U_t \hat{x}_t.
\end{align*}
Here, $\text{CH}(H_t+\lambda I, q_t, r)$ means $r$ iterations of CH to approximately solve $(H_t+\lambda_t I) x = q_t$. To improve stability and expressivity, next we allow regularization $\lambda$ and weight $\eta_i$ to be time-varying and chosen adaptively. We write $\lambda_t$ and $\eta_{t,i}$ to make their dependency in time $t$ explicit, with $\eta_{t,i}$ being the weight of the $i$-th token at time $t$. 

\begin{algorithm}[t]
    \SetAlgoLined
    \DontPrintSemicolon
    \texttt{Input}: $H\in \bbR^{D\times D},q\in \bbR^D$, eigenvalue bounds $L,\mu$ with $L\geq\mu>0$, number of iterations $r$;

    \texttt{Initialize}: $\rho \gets \frac{L-\mu}{L+\mu}$; $\omega_0=2$; the first two iterates $\xi_{-1} \gets 0, \xi_0 \gets \frac{2q}{L+\mu}$;
    
    \texttt{For Loop ($i=1,\dots,r$)}:
    \vspace{-0.2cm}
    \begin{align}
        \omega_i &\gets \frac{4}{4 - \rho^2 \omega_{i-1}} \tag{weight schedule}  \label{eq:weight-update} \\ 
        \xi_{i} &\gets \xi_{i-1} - \frac{2 \cdot \omega_i }{L+\mu}  (H \xi_{i-1} - q) \tag{grad descent} \label{eq:ch-gd} \\ 
        \xi_{i} &\gets \xi_i + (\omega_i - 1 ) ( \xi_{i-1} - \xi_{i-2})  \tag{momentum} \label{eq:ch-momentum}
    \end{align}
    
    \texttt{Output}: $\xi_{r}$
    
    \caption{Chebyshev Iteration to solve $H\xi = q$  } 
    \label{algo:Chebyshev}
\end{algorithm}

\myparagraph{Adaptive Regularization} As mentioned, the condition number $\kappa_t$ of $H_t +\lambda_t I$ has to be controlled for any method to be numerically stable.
We choose $\lambda_t$ to be proportional to the Frobenius norm $\| H_t \|_{\text{F}}$, that is to set $\lambda_t= a \cdot \| H_t \|_{\text{F}}$ for some constant $a>0$. An upper bound on $\kappa_t$ now ensures: 
\begin{align}\label{eq:kappa-ub}
    \kappa_t = \frac{\lambda_{\text{max}}(H_t) + \lambda_t}{\lambda_{\text{min}}(H_t) + \lambda_t} \leq \frac{\| H_t \|_{\text{F}} + \lambda_t }{ \lambda_t }= \frac{a+1}{a}.
\end{align}
Here $\lambda_{\text{max}}(H_t), \lambda_{\text{min}}(H_t)$ are the maximum and minimum eigenvalues of $H_t$, respectively. Given this choice of $\lambda_t$, we set $L=\| H_t \|_{\text{F}} + \lambda_t $ and $\mu=\lambda_t$ for \cref{algo:Chebyshev}.

\myparagraph{Adaptive Weighting (Gating)} We use weights $\eta_{t,i}$ that are exponentially decaying in time: For all $t\geq i$, we parameterize $\eta_{t,i} = \prod_{j = i+1}^t \gamma_j$, with each $\gamma_j \in [0,1]$ learnable. The fading weights encode the ``prior" of \textit{recency bias} that has been shown to exist in LLMs \cite{fang2025large} without even explicitly computing the weights from the query-key dot products as in \cref{eq:attn}. Similarly to \cref{eq:attn}, the weights on the residuals are now time-varying, but differently, the exponentially decay parameterization allows for linear-time implementation.

\myparagraph{Forward Recurrence} We now summarize our recurrence which arms CH with adaptive regularization and weighting: 
\begin{equation}\label{eq:ch-forward}
    \begin{split}
        H_t&= \gamma_t \cdot  H_{t-1} + k_t k_t^\top, \ \  U_t = \gamma_t \cdot U_{t-1} + v_t k_t^\top, \\
        y_t &= U_t \hat{x}_t, \quad \hat{x}_t = \text{CH}(H_t+\lambda_t I, q_t, r).
    \end{split} \tag{CH}
\end{equation}

\paragraph{Remark (canonical \(\beta_t\)-augmented variant).}
Subsequent work~\citep{chattopadhyay2026priming} augments the recurrence above
with a learned input-selectivity gate \(\beta_t \in [0,1]\) that modulates
the per-token write into both \(H_t\) and \(U_t\):
\[
H_t = \gamma_t \cdot H_{t-1} + \beta_t\, k_t k_t^\top,
\qquad
U_t = \gamma_t \cdot U_{t-1} + \beta_t\, v_t k_t^\top.
\]
The \(\beta_t \equiv 1\) variant studied in this paper is recovered as a
special case. All experiments in this paper use \(\beta_t \equiv 1\); the
\(\beta_t\)-augmented variant consistently improves long-context
performance, with gains widening at longer contexts (see
\citep[Section 8.1]{chattopadhyay2026priming} for motivation and ablations).
The released \ourshortname implementation adopts the \(\beta_t\)-augmented form
by default, and we recommend it as the canonical \ourshortname layer going forward.

\subsection{Chunk-wise Implementation}\label{section:CH}
Here, we describe our implementation for the forward + backward passes for \cref{eq:ch-forward}. For more details, see \cref{appendix: forward+backward CH}.

\subsubsection{Forward Pass}


Similarly to \cite{Dao-ICML2024-mamba2,Yang-ICML2024-gla,Yang-NeurIPS2024}, we now describe a \textit{chunk-wise} implementation for \cref{eq:ch-forward}. In \cref{eq:ch-forward}, given $U_t$ and $\hat{x}_t$, computing $y_t=U_t \hat{x}_t$ in a chunk-wise fashion is similar to that of \cref{eq:linear-attn}; also similar is the calculation of $H_t \xi_{i-1}$ as needed in \cref{eq:ch-gd}. For these we refer the reader to \cite{Yang-ICML2024-gla,Yang-NeurIPS2024} for details. Our algorithmic novelty here is a chunk-wise computational formula for $\| H_t\|_{\text{F}}$, presented next.

Let $T$ be the sequence length and $C$ the chunk size such that $N:=T/C$ is an integer. For $t=0,\dots,N-1$, write $[t]:=tC$. The core idea of a chunk-wise implementation is as follows. First, we compute and store the \textit{initial state} $H_{[t]}$ of every chunk. This gives us implicit access to $H_{[t]+c}$ via unrolling the recurrence of $H_t$ for $c$ steps and therefore allows us to carry out computation with $H_{[t]+c}$; for example, we can compute the matrix-vector product $H_{[t]+1} \xi$ via $H_{[t]}\xi + \gamma_1 k_{[t]+1}k_{[t]+1}^\top \xi$. This is without forming $H_{[t]+1}$ explicitly, thereby reducing the number of states to materialize on chip. To implement such a scheme, we need to precompute all $H_{[t]}$'s sequentially, and then do the computation with parallelism over chunks and within each chunk. 

We now make this idea precise for computing all $\| H_t \|_{\text{F}}$'s within a chunk. 
Since the computation of each chunk is the same, we simplify by working with the first one where we have access to initial state $H_0$,  gates $\gamma_1,\dots,\gamma_C$, keys $K_C=[k_1,\dots,k_C]\in\bbR^{D\times C}$,and we aim to compute $\| H_1 \|_{\text{F}},\| H_2 \|_{\text{F}}, \dots, \| H_C \|_{\text{F}}$. With these notations, we first compute the $C$-dimensional vector $\zeta=[\zeta_1,\dots,\zeta_C]^\top$ of cumulative products of $\gamma_i$'s, with $\zeta_{c}= \prod_{i=1}^c \gamma_{i}$. Then, form the $C\times C$ upper triangular matrix  $M$ whose ($i,j$)-th entry $M_{j,c}$ is $\zeta_{c} / \zeta_{j}$ ($\forall c\geq j$). Now, unroll the recurrence of $H_c$:
\begin{align*}
    H_{c} &= \zeta_{c} H_{0} + \sum_{j=1}^c M_{j,c} k_{j} k_{j}^\top =\zeta_{c} H_{0} + \sum_{j=1}^C M_{j,c} k_{j} k_{j}^\top. \\ 
\end{align*}
Expanding $\| H_{c} \|_{\text{F}}^2$ gives  the following sum of three terms:
\begin{align*}
    \zeta_{c}^2 \| H_{0} \|_{\text{F}}^2 + 2\zeta_c \sum_{j=1}^C M_{j,c}  k_{j}^\top H_0 k_{j}  +  \Big\| \sum_{j=1}^C M_{j,c} k_{j} k_{j}^\top \Big\|_{\text{F}}^2.
\end{align*}
With $\zeta$, the first term $\zeta_{c}^2 \cdot \| H_{0} \|_{\text{F}}^2$ is easily computed in parallel for all $c$. For the second term, we first compute the vector of quadratic forms $k_j^\top H_0 k_j$ for all $j$ in parallel, broadcast it and multiply it with $M$ element-wise, sum over each column, and multiply the result with $2\zeta$ element-wise. Finally, with Gram matrix $G_C:=K_C^\top K_C$, one verifies the third term can be computed in parallel for all $c$ via the following pseudocode:
\begin{align}
   \text{column-sum} (( (G_C \odot G_C) M  ) \odot M).
\end{align}
Here $\odot$ denotes element-wise multiplication and the sum is over each column. Summing the three terms and taking the square root, we obtain $\| H_1 \|_{\text{F}}, \dots, \| H_C \|_{\text{F}}$, as desired.




\subsubsection{Backward Pass}\label{eq:CH-backward-mainpaper}
\myparagraph{Motivation} Typically, the backward pass is done automatically via \texttt{torch.autograd}. However, for iterative methods such as CH (\cref{algo:Chebyshev}), \texttt{torch.autograd} would store some \textit{activations} or intermediate iterates, entailing large storage cost. While in principle we can back-propagate through CH without storing any intermediate activations or iterates (by our trick of \textit{reverting the CH iterations}, cf. \cref{subsection:CH-BP-details}), under this trick it is difficult to compute all the gradients in a chunk-wise fashion. Therefore, we resort to the \textit{implicit differentiation} trick, which is practically efficient and chunk-wise implementable, for backpropagation through the linear equations that CH approximately solves.

\myparagraph{Implicit Differentiation} We derive the backward pass for our method with the standard implicit differentiation trick. It assumes we find an exact solution $x^*_t$ to the equations $(H_t+\lambda_t I)x = q_t$. In the backward pass, we are given the gradient $dx^*_t:= \frac{d \cL}{d x^*_t}$ of some loss function $\cL$, and need to compute the corresponding gradients at $q_t, k_t,\gamma_t$. For example, via the chain rule we obtain $d q_t := \frac{d \cL}{d q_t}$ via 
\begin{align}\label{eq:exact-dq-ridge}
    d q_t= (H_t+\lambda_t I)^{-1}d x^*_t,
\end{align}
that is to solve linear equations similarly to the forward pass. Since the forward pass computes an approximate solution $\hat{x}_t$ via CH, we receive an approximate up stream gradient $d \hat{x}_t$ (not exactly $dx_t^*$). Thus we employ CH to obtain an approximate gradient $d\hat{q}_t = \text{CH}(H_t+\lambda_t I, d \hat{x}_t, r)$. See \cref{table:implicit-diff} for the full forward and backward equations under implicit vs. exact differentiation.

\myparagraph{Backward Recurrence} Besides $dq_t$, we need to compute $dH_t$ from which we obtain $dk_t$ and $d\gamma_t$ via the chain rule. We describe $d\gamma_t$ in the Appendix. Here we analyze $dk_t$:
\begin{lemma}\label{lemma:dk_i}
    With $\lambda_t=  a \| H_t \|_{\text{F}}$, $w_t=\frac{ 2a  \cdot (x_t^*)^\top  dq_t}{\| H_t \|_{\text{F}} } $, we have
    \begin{align}\label{eq:dk_i-main} 
        dk_i &=  \sum_{t\geq i} M_{i,t} \left(  -dq_t (x_t^*)^\top -  x_t^* dq_t^\top-  w_t H_t \right) k_i.
    \end{align}
\end{lemma}
With  $A_i:= \sum_{t\geq i} M_{i,t} \cdot  dq_t (x_t^*)^\top$, we can compute the first two terms $-A_ik_i$ and $-A_i^\top k_i$ in \cref{eq:dk_i-main}, similarly to \cref{eq:linear-attn}. Specifically, $A_i$ satisfies the recursion
\begin{align}\label{eq:Ai}
    A_i = \gamma_{i+1} A_{i+1} +  dq_i (x_i^*)^\top,
\end{align}
thus calculating $A_ik_i$ amounts to calculating $U_tq_t$ in \cref{eq:linear-attn}; a difference is that the recursion here runs backwards. 

Similarly, with $ B_i = \sum_{t\geq i}  M_{i, t} w_t H_t$, the third term in \cref{eq:dk_i-main} can be written recursively as 
\begin{align}\label{eq:Bi}
    B_i = \gamma_{i+1} B_{i+1} + w_i H_i, \quad o_i =  B_i k_i.
\end{align}

\myparagraph{Chunk-wise Recurrence} As indicated, a chunk-wise implementation for computing $A_ik_i$ is known. On the other hand, computing $B_ik_i$ is more challenging than $A_ik_i$, as the additive term $w_iH_i$ in the backward recursion \cref{eq:Bi} is not necessarily rank-$1$; rather, $H_i$ itself is defined via the forward recursion in \cref{eq:ch-forward}. Our contribution here is a derivation for computing $B_ik_i$ efficiently in a  chunk-wise manner.

We begin by unrolling $B_i$ to $B_{C+1}$:
\begin{equation}\label{eq:Bi=intra+inter}
     \begin{split}
         B_i =& M_{i, C} \cdot \gamma_{C+1}B_{C+1} + B_{i}^{\text{intra}} \\ 
     B_{i}^{\text{intra}}:=&  \sum_{c=i}^{C} M_{i,c} \cdot w_c H_c
     \end{split}
\end{equation}
We next discuss the \textit{intra-chunk} term $B_{i}^{\text{intra}} k_i$ and \textit{cross-chunk} term $M_{i, C} \cdot \gamma_{C+1}B_{C+1}$ in succession.

\myparagraph{Intra-chunk Computation} We now unroll $H_c$ and obtain an expression $B_{i}^{\text{intra}}$ more amenable to parallelism:
\begin{align*}
    B_{i}^{\text{intra}}  &= \sum_{c=i}^{C} M_{i,c} \cdot w_c H_c \\ 
    &= \sum_{c=1}^{C} M_{i,c}  w_c  \Big( \zeta_c H_0 + \sum_{j=1}^C M_{j,c} k_j k_j^\top  \Big) \\ 
    &= H_0 \sum_{c=1}^{C} M_{i,c}  w_c  \zeta_c  + \sum_{j=1}^C k_j k_j^\top \sum_{c=1}^{C} M_{i,c} w_c M_{j,c} .  
\end{align*}
The coefficients of $H_0$, written as $b_i$, are easily computed in parallel for all $i$ via element-wise operations, broadcasting, and summing. The coefficient of $k_j k_j^\top$ is precisely the $(i,j)$-th entry of the matrix $M_w:= M\diag(w_1,\dots,w_C)  M^\top $. Thus  $[B_{1}^{\text{intra}}k_1,\dots,B_{C}^{\text{intra}}k_C]$ is equal to
\begin{align*}
    H_0 (\diag(b_1,\dots,b_C) K_C) +  K_C \left( (K_C^\top K_C) \odot M_w \right).
\end{align*} 
The \textit{mask} $M_w$ is in general a full matrix with no zero entries, as opposed to the triangular matrix in the case of \cref{eq:linear-attn}. While the triangular mask in the backward pass allows the error feedback from future tokens to be leveraged for learning past tokens, here our full mask $M_w$ allows all tokens to interact with all other tokens in the backward pass, which facilitates the information flow and learning.

\begin{figure*}[h]
    \centering
    \includegraphics[width=0.195\textwidth]{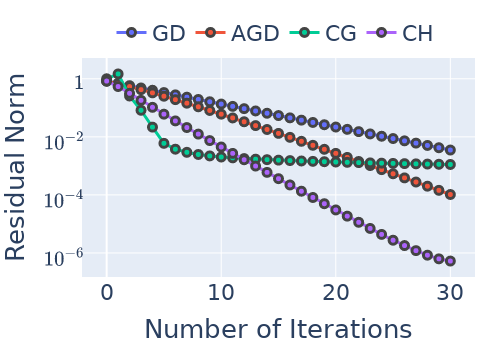}
    \includegraphics[width=0.195\textwidth]{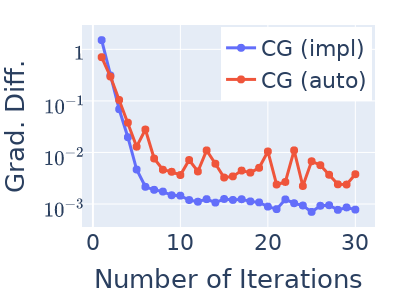}
    \includegraphics[width=0.195\textwidth]{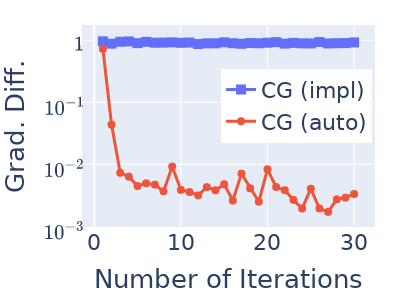}
    \includegraphics[width=0.195\textwidth]{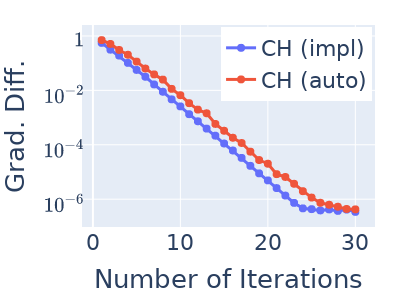} 
    \includegraphics[width=0.195\textwidth]{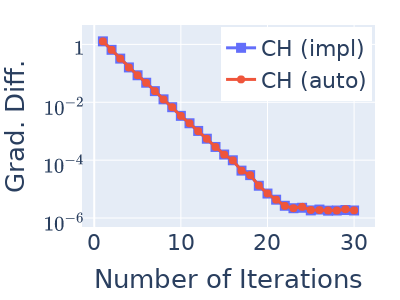}

    \makebox[0.195\textwidth]{\footnotesize   (a) Empirical convergence}
    \makebox[0.195\textwidth]{\footnotesize  (b) CG as a single layer}
    \makebox[0.195\textwidth]{\footnotesize  (c) CG in 5-layer LLAMA }
    \makebox[0.195\textwidth]{\footnotesize  (d) CH as a single layer}
    \makebox[0.195\textwidth]{\footnotesize  (e) CH in 5-layer LLAMA }
    
    \caption{\textbf{CH converges with smaller errors than CG and is more numerically stable.} Convergence of different methods  in residual norms during the forward pass with batch size $8$, sequence length 2048, 8 heads,  head dimension 128 (a), and relative gradient differences  from the exact solver (\texttt{torch.linalg.solve}) to CG (b, c) or CH (d, e). The backward pass is via \textit{implicit differentiation} (\textit{impl}) or \texttt{torch.autograd} (\textit{auto}); cf. \cref{table:implicit-diff}.  In (b, d) the gradients are those of $[q_t,k_t]$; in (c, e) the gradients are those of network weights. \label{fig:cgch-grad}}
    \vspace{-0.3cm}
\end{figure*}

\myparagraph{Cross-chunk Computation} In \cref{eq:Bi=intra+inter}, both $\gamma_{C+1}$ and $B_{C+1}$ are from the future chunks, thus we revise \cref{eq:Bi=intra+inter} into the cross chunk recursion of $\widetilde{B}_{C+1}:=\gamma_{C+1}B_{C+1}$ which allows us to maintain a single term $\widetilde{B}_{C+1}$ from the future:
\begin{align*}
    \widetilde{B}_1 = \zeta_C \cdot \widetilde{B}_{C+1} + \widetilde{B}^{\text{intra}}_1, \ \  \ \widetilde{B}^{\text{intra}}_1:=\sum_{c=i}^{C} \zeta_c \cdot w_c H_c.
\end{align*}
In our intra-chunk computation, we store the  intra-chunk term $\widetilde{B}^{\text{intra}}_1$ of all chunks, implement the above with a simple for loop, and collect the terms $\zeta_C \cdot \widetilde{B}_{C+1}k_i$.

\subsubsection{Comparison to Other Iterative Solvers}\label{subsection:solver-comparison}
Here we validate our choice of Chebyshev Iteration (CH) by benchmarking it against other iterative methods. 

\myparagraph{Convergence in the Forward Pass} We generate random regression problems, which we solve via CH and 3 other baselines:  gradient descent (GD), accelerated GD with Nesterov's momentum (AGD), conjugate gradient (CG). GD and AGD are run with stepsizes that are optimal for regression problems. \cref{fig:cgch-grad}a shows CG converges the fastest within a few iterations, while CH reaches the same accuracy as CG at iteration 10 and eventually attains the smallest errors.

\begin{table}[]
    \centering
    \caption{\textbf{Implicit differentiation for computing $dq_t$.} 
    }
    \label{table:implicit-diff}
    \scalebox{0.8}{
        \begin{tabular}{lcc}
        \toprule 
          & forward pass & backward pass \\
          \midrule
           exact & $x^*_t = (H_t+\lambda_t I)^{-1} q_t$ & $dq_t^*= (H_t+\lambda_t I)^{-1} dx^*_t$  \\
           CG  & $\hat{x}_t = \text{CG}(H_t+\lambda_t I, q_t, r)$ & $ d\hat{q}_t = \text{CG}(H_t+\lambda_t I, d \hat{x}_t, r)$\\ 
           CH & $\hat{x}_t = \text{CH}(H_t+\lambda_t I, q_t, r)$ & $ d\hat{q}_t = \text{CH}(H_t+\lambda_t I, d \hat{x}_t, r)$\\ 
           \bottomrule
        \end{tabular}
    }
\end{table}

\myparagraph{Stability of the Backward Pass} We then proceed and measure the gradient stability of CG and CH, whose backward passes are implemented either via implicit differentiation as per \cref{table:implicit-diff} (\textit{impl}), or via \texttt{torch.autograd} (\textit{auto}).

In \cref{fig:cgch-grad}b, CG (impl) as a standalone layer has its gradient close to that of the exact solver up to a $10^{-3}$ relative difference. In \cref{fig:cgch-grad}c, this difference is amplified to almost $1$ in a 5-layer LLAMA where \cref{eq:attn} is replaced with \cref{eq:rls}. This indicates CG (impl) completely deviates from the reference gradient (exact), defeating its purpose of training the network from the regression feedback. In contrast, the gradients of CH (impl) and CH (auto) are eventually close to that of the exact solver either as a single layer (\cref{fig:cgch-grad}c) or within multiple layers (\cref{fig:cgch-grad}d), up to a $10^{-6}$ difference. Moreover, the curves for CH (impl) and CH (auto) nearly overlap, suggesting that their gradients may be close. \cref{lemma:CH-autoq=CH-implq} formalizes this intuition and justifies our choice of CH over the alternatives:
\begin{lemma}\label{lemma:CH-autoq=CH-implq}
    Let $dq_t$ be the exact gradient of $q_t$ for CH, e.g., computed by CH (auto). Let $d \hat{q}_t$ be the gradient of CH (impl), computed as per \cref{table:implicit-diff}. We have $dq_t = d\hat{q}_t$.
\end{lemma} 
See \cref{appendix:proof-implicit-explicit-dq} for a proof.


\begin{figure}
    \centering
    \includegraphics[width=0.28\textwidth]{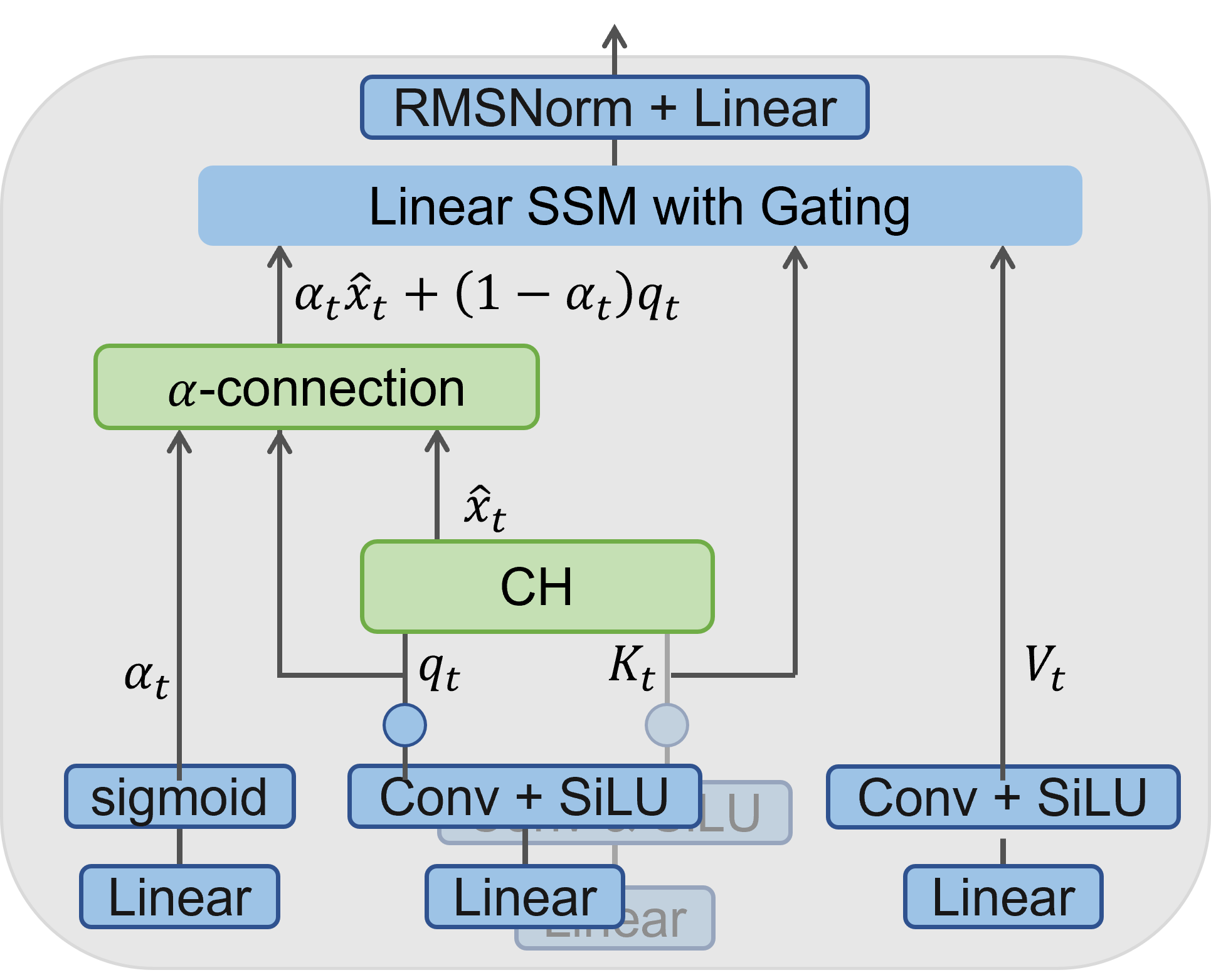}
    \caption{\textbf{Our \ourshortname block}. Blue refers to established practices in the literature with the solid circles denote $\ell_2$ normalization. Green components (CH and $\alpha$-connection) are our proposals. }
    \label{fig:GKA-block}
    \vspace{-0.5cm}
\end{figure}


\begin{figure*}
    \centering
    \includegraphics[width=0.7\textwidth]{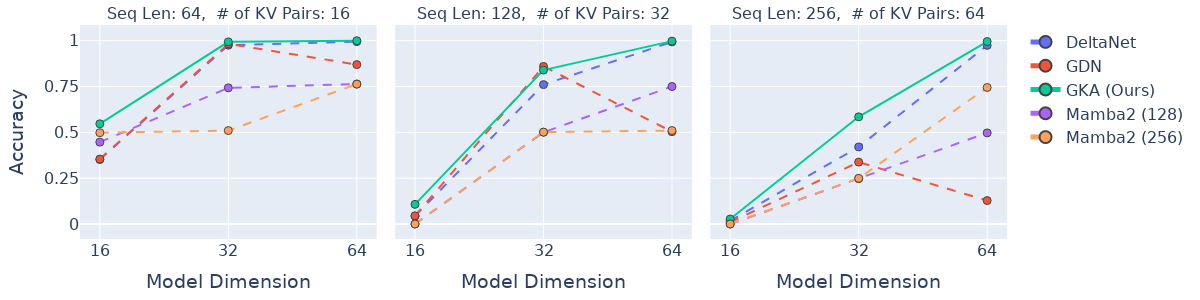}
    \includegraphics[width=0.28\textwidth]{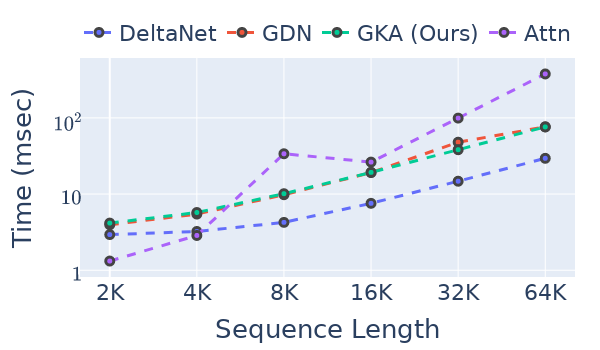}

   \makebox[0.7\textwidth]{\footnotesize  (a) Accuracy vs. model dimension for different fading memory layers on MQAR. }
    \makebox[0.28\textwidth]{\footnotesize \quad  \quad (b) Runtime of a single memory layer  }
    
    \caption{ \textbf{MQAR results} (a) Each plot corresponds to a particular sequence length and number of key-value pairs for the model to memorize. \textbf{Runtime} (b)  Runtimes are for a single forward + backward pass (8 heads, head dim $128$, batch size $4$, averaged over 20 runs). }
    \label{fig:mqar+fig:speed}
    \vspace{-0.3cm}
\end{figure*}

\subsection{Architectural Consideration}\label{subsection:GKA-block}
Our \ourshortname layer in \cref{fig:GKA-block} includes two components (in green) on top of established practices (in blue). The CH component is  described in \cref{subsection:CH-adaptive-reg}, thus here we introduce the \textit{$\alpha$-connection}.  First,  the sigmoid activation ensures $\alpha_t\in[0,1]$, so the output of the $\alpha$-connection is a convex combination of  the original query $q_t$ and the output $\hat{x}_t$ of CH. Second, it plays a similar role to residual connection, which establishes a direct path that facilitates the gradient flow and improves training; we show this is indeed the case in \cref{appendix: alpha connection ablation}. Finally, the full architecture for \ourshortname is the standard Transformer, with its attention layer replaced by the \ourshortname layer.


\section{Experiments}
\label{section: experiments}

In this section, we empirically validate the efficacy of \ourshortname. We first evaluate memorization ability on synthetic associative recall tasks (\cref{subsection: MQAR}). We then report training throughput of \ourshortname  (\cref{subsection: throughput}). Next, we examine performance on short-context language understanding benchmarks such as commonsense reasoning and long-context modeling abilities in \cref{subsection: language modeling}. Finally, we demonstrate the effectiveness of GKA beyond language modeling by evaluating on ImageNet classification (\cref{subsection: vision}). The Appendix details our experimental settings (\cref{sec: Model architecture}) and ablations of various modeling choices (\cref{appendix: ablations}, \cref{appendix: ablation-reg-strength}).

\myparagraph{Baselines} All experiments consider the following state-of-the-art linear SSM-based fading memory layers as baselines: Mamba2 \cite{Dao-ICML2024-mamba2}, DeltaNet \cite{Yang-NeurIPS2024}, Gated DeltaNet (GDN) \cite{Yang-ICLR2025}, and Gated Linear Attention (GLA) \cite{Yang-ICML2024-gla}. Each of these layers rely on instantaneous objectives that depend on the previous \textit{lossy} state and current tokens (e.g., \cref{eq:GDN-obj}), as opposed to the entire history of tokens observed so far as in \ourshortname. Finally, we contrast our results with (Softmax) Attention, which serves as our paragon.
For our Attention-based model, we adopt the architecture proposed in Qwen3 models \cite{yang2025qwen3}.

\subsection{\ourshortname on Synthetic Associative Recall Tasks} \label{subsection: MQAR}


We first assess the capability of our models to recall information on the multi-Query Associative Recall (MQAR) task ~\cite{arora2023zoology}.
This task presents the model with a sequence of key-value pairs to memorize, followed by a sequence of queries. For each query, the model must retrieve the corresponding key from memory and accurately recall its associated value.
Attention based layers perform the best in this task, while SSM-based memory layers are known to struggle as their memory fades away as the context length grows. 

We compare \ourshortname with Attention and other linear SSM baselines on this task. For each memory layer type, we train 2-layer models on MQAR training data and evaluate on a held-out test set. We repeat this experiment for $4$ different learning rates spanning from $10^{-4}$ to $10^{-2}$. 
As shown in \cref{fig:mqar+fig:speed}a, \ourshortname improves upon every other linear SSM baseline at all sequence lengths and model dimensions considered. Note, the complexity of the task increases with increasing sequence length and number of key-value pairs, while larger model dimensions improve memorization capacity through increased state size. The success of our layer can be attributed to our modeling choice: unlike other fading memory designs (like GDN or Mamba2), we construct states based on the optimal MAP estimate conditioned on the entire history, enabling better retention of remote information.

\subsection{Training Throughput of \ourshortname}
\label{subsection: throughput}
\cref{fig:mqar+fig:speed}b compares the running time (forward + backward) of a single \ourshortname layer against FlashAttention~\cite{golden2024flash}, DeltaNet, and GDN (see also \cref{section:throughput-long-seq}). Despite its more expensive state update \cref{eq:ch-forward}, \ourshortname achieves a speed comparable to that of GDN at the same state size (GDN expands the value dimension by 2$\times$ by default). This demonstrates the efficiency of our
chunk-wise parallelization strategy which effectively compensates
for the additional computational cost.


\begin{table*}[htbp]
\centering
\caption{\textbf{On average \ourshortname improves upon all fading memory baselines across all tasks.} We report results for zero-shot evaluation of 2.8B language models for short-context tasks. For each task, bold indicates highest value followed by underlined.}
\label{table:ssm_lm_harness_comparison}
\small
\resizebox{0.9\textwidth}{!}{
\begin{tabular}{l|cccccccccc|cc}
\toprule
\textbf{Model} & \textbf{\begin{tabular}{@{}c@{}}ARC-C\end{tabular}} & \textbf{\begin{tabular}{@{}c@{}}ARC-E\end{tabular}} & \textbf{BoolQ} & \textbf{COPA} & \textbf{HellaSWAG} & \textbf{PIQA} & \textbf{SciQ} & \textbf{Winogrande} & \textbf{FDA} & \textbf{SWDE} & \textbf{Avg} \\
& acc\_n $\uparrow$ & acc\_n  $\uparrow$ & acc $\uparrow$ & acc $\uparrow$ & acc\_n $\uparrow$ & acc\_n $\uparrow$ & acc\_n $\uparrow$ & acc $\uparrow$ & contains $\uparrow$ & contains $\uparrow$ & \\
\midrule
\midrule
Transformer & 32.25 & 56.10 & \textbf{64.28} & 80.00 & 60.96 & 73.56 & 79.50 & 61.72 & \textbf{58.53} & \textbf{72.28} & \textbf{63.92} \\
\midrule
Gated Linear Attention & 27.82 & 50.80 & 52.57 & 78.00 & 48.83 & 70.13 & 69.60 & 54.54 & 2.81 & 20.43 & 47.55 \\
DeltaNet & \textbf{32.85} & 58.16 & 42.51 & 81.00 & 61.13 & 73.78 & 43.90 & 61.72 & 11.80 & 46.08 & 51.29 \\
Mamba2 & 32.24 & 59.64 & 58.72 & \underline{82.00} & 62.23 & 73.78 & 79.80 & 62.19 & 7.71 & 41.13 & 55.94 \\
Gated DeltaNet & \underline{32.59} & \textbf{60.02} & \underline{62.75} & \underline{82.00} & \underline{62.80} & \underline{74.32} & \underline{80.60} & \underline{62.35} & 8.26 & 44.28 & 57.00 \\
Gated KalmaNet (Ours) & 32.51 & \underline{59.89} & 61.68 & \textbf{85.00} & \textbf{63.84} & \textbf{74.81} & \textbf{83.20} & \textbf{64.17} & \underline{12.89} & \underline{50.95} & \underline{58.89} \\
\bottomrule
\end{tabular}
}
\end{table*}

\begin{figure*}
    \centering
    \includegraphics[width=1.0\textwidth]{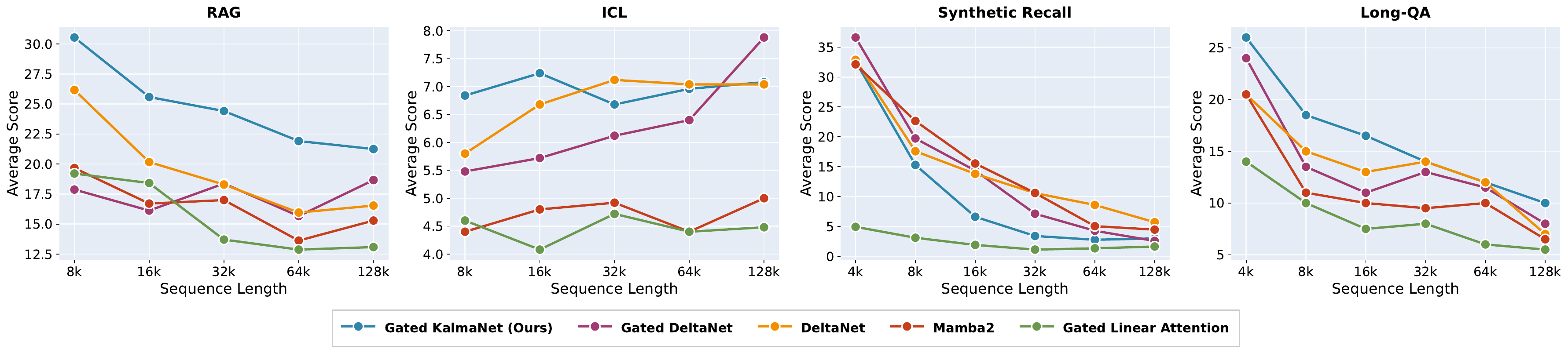}
    \caption{\textbf{Long Context Performance up to 128k tokens}. \ourshortname outperforms baseline by 10\% in relative improvement on RAG and LongQA. There is no clear winner Synthetic Recall. All models struggle to perform better than random chance on ICL.}
    \label{fig:longctx_ssm}
    \vspace{-6pt}
\end{figure*}
    
\subsection{\ourshortname on Language Modeling} 
\label{subsection: language modeling}
\subsubsection{Short-context Tasks}
\label{subsubsection: short-context tasks}
\myparagraph{Setup} For this set of experiments, we construct 2.8B LLM models for each memory layers (\ourshortname and baselines described in \cref{section: experiments}) by cascading blocks of \texttt{mem} + Multi-Layer Perceptron (MLP) blocks.\footnote{For Mamba2 baseline, we consider cascading blocks of Mamba2 layer alone since a single Mamba2 layer has the Mamba2 SSM and MLP.} Hereby, we refer to the 2.8B models with the same name as the layer used to construct them. We then train each model on DCLM \cite{li2024datacomp}, a generic pre-training dataset for $100$B tokens at $4$K context length using the AdamW optimizer with a peak Learning Rate (LR) of $10^{-3}$ and gradient clipping of $1.0$. We used the cosine LR scheduler with a warmup period of $5$B tokens with a global batch size of $2$M tokens. All models employ the GPT2 tokenizer with a vocabulary size of $50$K tokens. 

\myparagraph{Tasks} Following prior works \cite{Zancato-NeurIPS2024-BMOJO, Yang-ICML2024-gla, Yang-ICLR2025}, to consider language modeling capabilities of our model we perform zero-shot evaluation on the following eight common-sense reasoning tasks from LM-Harness \cite{eval-harness}: Arc-E, Arc-C, BoolQ, COPA, HellaSWAG, PIQA, SciQ, Winogrande. We also evaluate models on FDA and SWDE, real-world recall-intensive tasks which focus on extracting structured information like tagged content from raw text (for example, HTML files). All these tasks are relatively short ($<2$K tokens).

\myparagraph{Results} We report our results in \cref{table:ssm_lm_harness_comparison}. \ourshortname outperforms all fading memory baselines on average across all tasks owing to its ability to better manage its state via solving \cref{eq:rls}. In particular, \ourshortname outperforms both GDN and Mamba2 on recall-intensive tasks (FDA and SWDE) by about $10\%$ (rel. improvement). We note that although \ourshortname improves upon existing SSM layers there is still a gap with Attention-based Transformer especially on recall-tasks owing to the eidetic capabilities of Attention. Nevertheless, as discussed in \cref{section:intro} this improvement comes at a quadratic cost at training time, whereas our layer's computational complexity is still comparable to existing SSM layers (cf. \cref{subsection: throughput}). In \cref{appendix: hybrid} we extend our results to Hybrid models (stack of SSM and Attention layers) and show that the gap with full Transformer models becomes negligible (while still benefiting the SSM's computational advantages). Finally, in \cref{appendix: GKA scaling laws} we show that \ourshortname exhibits stronger scaling with compute than other SSM baseline models.


\subsubsection{Long-context Tasks}

\myparagraph{Setup} To enable long-context capabilities of our models, as is common practice, we perform continued pre-training of our 2.8B models obtained in \cref{subsubsection: short-context tasks} on $25$B tokens of long documents at $128$K context length (cf. Appendix). To the best of our knowledge we are the first to train and evaluate SSM models up to $128$K context (e.g., previous work \cite{Yang-ICLR2025} only considered up to $4$K/$8$K context).

\myparagraph{Tasks} For long-context, we refrain from using perplexity as it is known to have limitations at assessing long-context performance of LLMs \cite{nunez2024expansion,fang2024wrong,gao2025train}. Instead, we turn to recently proposed benchmarks that mix synthetic and real datasets comprising several long-context tasks: Synthetic Recall, Retrieval-Augmented Generation (RAG), Many shot In-Context Learning (ICL) and Long Question-Answering (LongQA). For Synthetic Recall and LongQA we consider tasks from the RULER benchmark \cite{hsieh2024ruler}. For RAG and ICL we consider tasks from HELMET \cite{yen2025helmet}. 

\myparagraph{Results} \cref{fig:longctx_ssm} reports our results. \ourshortname shows strong RAG and LongQA capabilities, outperforming all fading memory baselines by at least $10$\% (rel. improvement). Interestingly, on Synthetic Recall tasks from RULER, \ourshortname is competitive only at $4$K context length and starts to fall behind afterwards. We attribute this gap to the nature of the tasks. RAG and LongQA involve realistic linguistic patterns that align with natural text distributions seen during pretraining, whereas synthetic recall requires verbatim retrieval of random tokens from artificial contexts. Since \ourshortname computes MAP estimates of the latent state using learned representations, its pretrained weights guide which information to retain. This is advantageous in naturalistic settings, but less so when the signal-to-noise distinction is purely artificial (see \cref{section:no-S-NIAH} for further analysis).

\subsection{\ourshortname on Image Classification}
\label{subsection: vision}
SSMs have gained popularity in vision tasks \cite{zhuvision, hatamizadeh2025mambavision, liulonghorn}. Here, we investigate \ourshortname as a viable alternative to other SSMs for vision problems. Specifically, we consider ImageNet classification \cite{Deng-CVPR2009} and compare against MambaVision \cite{hatamizadeh2025mambavision}, a state-of-the-art Hybrid Mamba-Transformer vision model. MambaVision modifies the vanilla Mamba layer with non-causal convolutions and a symmetric non-SSM branch, yielding about 2\% improvement over vanilla Mamba \cite[Table 4]{hatamizadeh2025mambavision}.We introduce our variant, GKAVision, obtained by replacing the MambaVisionMixer blocks in MambaVision with our \ourshortname layer. More details are provided in \cref{sec: GKAVision}. Our results in \cref{table:vision_imagenet}, show that GKAVision outperforms MambaVision while closing the gap with a pure vision Transformer (NextViT-S \cite{li2022next}) at 33\% higher training throughput. We did not implement any vision-specific changes to \ourshortname, which might result in further gains.

\begin{table}[htbp]
\centering
\caption{\textbf{GKAVision outperforms MambaVision on ImageNet classification} (averaged over 5 seeds). Models are $\sim$31.8M parameters; training throughput measured on 8 H200 GPUs with batch size 256. Unlike MambaVision, GKAVision uses the GKA layer without any vision-specific modifications.}
\label{table:vision_imagenet}
\small
\resizebox{0.45\textwidth}{!}{
\begin{tabular}{l|c|c}
\toprule
\textbf{Model} & \textbf{Top-1 Accuracy (\%) $\uparrow$} & \textbf{Throughput (K img/s) $\uparrow$} \\
\midrule
\midrule
MambaVision-T & 81.18 & \textbf{16.25}\\
GKAVision-T & \underline{81.27} & \underline{13.72} \\
NextViT-S & \textbf{81.99} & 10.32 \\
\bottomrule
\end{tabular}}
\end{table}

\section{Kalman Filter for \textit{Optimally} Modelling Fading Memory}
In this section, we show how the Kalman Filter (KF) provides a principled solution for constructing an optimal fading memory that accounts for the entire history. We begin by describing the standard Kalman Filter recurrence in the context of memory modeling. However, the KF has a fundamental limitation: its inherently sequential nature makes it impractical for large-scale training on modern hardware accelerators (\cref{subsection:hurdle}). To address this, we make simplifying assumptions that makes KF amenable to parallelization on modern hardware accelerators. We then demonstrate that several recent state-space models (DeltaNet, Gated DeltaNet, and Kimi Delta Attention) can be viewed as approximations to the KF recurrence. Specifically, these methods approximate the ``optimal" Kalman gain matrix while ignoring dependencies on the past. In contrast, \ourshortname computes the exact Kalman gain by considering the full history. This theoretical advantage translates to improved empirical performance, as we demonstrate in \cref{section: experiments}.

\subsection{A Dynamical System for Fading Memory}
The Kalman filter is a classical algorithm for online optimal inference in Linear Gaussian State-Space Models. It gives a principled way to maintain and update an estimate of a latent state as new noisy observations arrive. The true latent state serves as a compressed 'memory' of the past: given the state, the future becomes conditionally independent of past observations.


We begin by describing a linear Gaussian model for fading memory.
\begin{equation}\label{eq: LGM}
\begin{aligned}
     s_t &= A_t s_{t - 1} + B_tu_t + w_t, \quad& w_t \sim \mathcal{N}(0, Q_t) \\
     v_{t} &= k_t^{\top}s_t + \mu_t, \quad&  \mu_t \sim \mathcal{N}(0, r_t),
\end{aligned}\tag{LGM}
\end{equation}
where $s_t \in \mathbb{R}^n$ is a latent state that summarizes the past, $u_t \in \mathbb{R}^n$ is the control input that updates the state and $v_t$ is the scalar measurement observed at time $t$. $A_t, B_t \in \mathbb{R}^{n \times n}$ are the state transition and input selection matrices, and $k_t \in \mathbb{R}^{n}$ is the emission (readout) vector. Finally, $w_t$ and $\mu_t$ are Gaussian process and measurement noise, respectively.

\textbf{Parameter interpretation.} $A_t$ and $B_t$ control the forgetting (fading of the remote past) and input selectivity rates respectively, determining how the state evolves over time. The measurement noise $\mu_t$ naturally gives rise to gating mechanisms commonly used in modern SSM layers, as we will show in \cref{sec: KF-connection-with-existing-ssms}.

\textbf{Extension to multi-channel measurements.}
In attention mechanisms, the memory consists of verbatim key-value pairs that can be queried to retrieve past information \cite{Zancato-NeurIPS2024-BMOJO}. Similarly, we want our state to reconstruct past values from their corresponding keys. To achieve this, we extend to a matrix-valued state $S_t \in \mathbb{R}^{n \times n}$, where each column independently follows the dynamics in \cref{eq: LGM}.

Specifically, for the $i^{\text{th}}$ channel:
\begin{equation*}\label{eq: LGM2}
\begin{aligned}
     s_{t,i} &= A_{t,i} s_{t - 1,i} + B_{t,i}u_{t,i} + w_{t,i}, \quad& w_t \sim \mathcal{N}(0, Q_{t,i}) \\
     v_{t,i} &= k_t^{\top}s_{t,i} + \mu_{t,i} \quad&  \mu_t \sim \mathcal{N}(0, r_{t,i}),
\end{aligned}
\end{equation*}
where $(k_t, v_t)$ is the key-value pair at time $t$ and $v_{t,i}$ is the $i^{\text{th}}$ element of $v_t$. In what follows, we focus on a single channel and drop the subscript $i$ from the state for notational clarity.

\subsection{Kalman Filter for Optimal Inference}
Given the model in \cref{eq: LGM} and a sequence of measurements $\{v_1, v_2, \ldots, v_t\}$, the Kalman Filter computes the \textit{Maximum A-Posteriori} (MAP) estimate of the latent state at time $t$:
\begin{equation}\label{eq: KF-MAP}
    \hat{s}_t = \arg\max_s p(s \mid v_1, v_2, \ldots, v_t),
\end{equation}
where $p$ is a probability density function. The MAP estimate is optimal in the sense that it minimizes the expected squared error between the true state and its estimation given all measurements up to time $t$.

\textbf{The KF recursion.} The Kalman Filter updates the state estimate recursively as new measurements arrive. At time $t$, the update is:
\begin{equation}\label{eq:KF-recursive-for-LGM}
    \hat{s_t} = \underset{\textrm{Predicted state}}{\underbrace{A_t \hat{s}_{t-1} + B_t u_t}} + G_t(\overbrace{v_{t,i} - k_t^{\top}\underset{\textrm{Predicted state}}{\underbrace{\Big[A_t \hat{s}_{t-1} + B_t u_t\Big]}}}^{\textrm{Innovation}}),\\
\end{equation}
where the \textit{innovation} measures the discrepancy between the actual measurement $v_t$ and the predicted measurement based on the predicted state estimate. 

The \textit{Kalman gain} $G_t$ determines how much to trust the new measurement versus the predicted state. It is computed as follows:
\begin{equation}\label{eq: kalman gain update}
    G_t = \frac{\Big[A_t \Sigma_{t - 1} A_t^T + Q_t\Big]k_t}{k_t^{\top}\Big[A_t \Sigma_{t - 1} A_t^T + Q_t\Big]k_t + r_t}.
\end{equation}
The \textit{error covariance} $\Sigma_t$ quantifies the uncertainty in the state estimate. It represents the covariance of the estimation error $(s_t - \hat{s}_t)$ conditioned on all measurements up to time $t$. The covariance is updated as:
\begin{equation}\label{eq: error covariance update}
    \Sigma_{t} = \Big(I - G_tk_t^{\top}\Big)\Big(A_t \Sigma_{t - 1} A_t^T + Q_t\Big)
\end{equation}

Equations \eqref{eq:KF-recursive-for-LGM}, \eqref{eq: kalman gain update} and \eqref{eq: error covariance update} constitute the KF recursion. We initialize with $\hat{s}_0 = 0$ and $\Sigma_0 = \sigma I_n$, where $I_n$ is the $n \times n$ identity matrix and $\sigma$ represents our prior uncertainty about the state before observing any measurements.

\subsection{\ourname: A Steady-State Dynamical System for Large-Scale Training}
Despite its optimality, the KF recursion in its most general form is inherently sequential; each update depends on the previous state estimate. This sequential dependency prevents the parallelization necessary for efficient large-scale training on modern hardware.

To enable parallelization, we make a key simplifying assumption: the underlying state remains static over time. This reduces the problem from \textit{tracking} a dynamic state to \textit{estimating} a fixed but unknown parameter from sequential noisy measurements. Formally, we assume a steady-state model:
\begin{equation}\label{LGM-steady-state}
\begin{aligned}
     s_t &= s_{t - 1} \\
     v_{t,i} &= k_t^{\top}s_t + \mu_t, \quad&  \mu_t \sim \mathcal{N}(0, r_t),
\end{aligned}
\end{equation}
where $A_t = I_n$, $B_t = 0$, and $w_t = 0$ (i.e., no state evolution, no control input, and no process noise).

\textbf{Adapting to evolving context.} While the steady-state assumption may initially seem restrictive, contexts naturally evolve as topics change, \ourshortname addresses this through adaptive weighting (\cref{subsection:CH-adaptive-reg}). By assigning higher weights to recent measurements, older observations are naturally faded out over time, allowing the model to track shifting context despite the static formulation.

Under this simplification, the KF recursion reduces to:
\begin{equation}\label{eq:KF-recursive-for-LGM-steady-state}
\begin{aligned}
    \hat{s}_t = \hat{s}_{t-1} + G_t(v_{t,i} - k_t^\top\hat{s}_{t-1}).\\
     G_t = \frac{\Sigma_{t - 1}k_t}{k_t^{\top} \Sigma_{t - 1}k_t + r_t}.\\
    \Sigma_{t} = \Big(I - G_tk_t^{\top}\Big)\Sigma_{t - 1}.
\end{aligned}
\end{equation}
Collecting all channels, these equations can be written compactly in matrix form as shown in \eqref{eq:kf-update}\footnote{with columns of $S_t$ transposed to being rows of $S_t$ to be consistent with the notation in \eqref{eq:kf-update} and taking the noise variance $r_t = \frac{1}{\eta_t}$.}. A key insight of this work is that the KF recursion for the steady-state model admits an efficient parallel implementation via chunked processing (detailed in \cref{sec: GKA}) that results in \ourname. 

Critically, the KF recursion accounts for the entire history when computing state estimates. The Kalman gain $G_t$ at each step depends on all previous measurements through $\Sigma_{t-1}$. This contrasts with most existing SSMs, which we show next can be viewed as approximations that ignore historical dependencies when computing their gain matrices. This principled treatment of the full history is a key advantage of our approach.

\subsection{Connection with Existing SSM Layers}
\label{sec: KF-connection-with-existing-ssms}
\noindent\textbf{DeltaNet} \cite{Yang-DeltaNet}  approximates the KF recursion in \eqref{eq:KF-recursive-for-LGM-steady-state} by assuming fixed error covariance: $\Sigma_t = I_n$ for all $t$. This simplifies the Kalman gain to:
\begin{equation}\label{G_t for deltanet}
    G_t = \frac{k_t}{k_t^{\top}k_t + r_t} = \frac{k_t}{1 + r_t},
\end{equation}
where the second equality assumes unit-normalized keys, a common assumption in practical instantiations of DeltaNet. Substituting \eqref{G_t for deltanet} into the state update \eqref{eq:KF-recursive-for-LGM-steady-state} and defining $\beta_t~=~(1 + r_t)^{-1}$ yields:
\begin{equation}\label{eq: DeltaNet}
    \hat{s}_t = (I - \beta_tk_tk_t^{\top})\hat{s}_{t-1} + \beta_tk_tv_{t,i},\tag{DeltaNet}
\end{equation}
which is the DeltaNet recurrence. By fixing $\Sigma_t$, DeltaNet avoids tracking the evolving uncertainty in the state estimate, a key simplification that sacrifices optimality for computational efficiency. In contrast, \ourshortname maintains the full error covariance $\Sigma_t$, allowing it to optimally weight measurements based on the entire history.

\noindent\textbf{Gated DeltaNet (GDN)} \cite{Yang-ICLR2025} extends DeltaNet by incorporating explicit forgetting through a time-dependent decay factor $\alpha_t$. Like DeltaNet, GDN can be viewed as fixing $\Sigma_t = I_n$, but applying this approximation to the KF recursion for a fading dynamical system where the state decays over time.

Specifically, GDN assumes
\begin{equation}\label{LGM-fading-state}
\begin{aligned}
     s_t &= \alpha_t s_{t - 1} + w_t \quad&  w_t \sim \mathcal{N}(0, I_n)  \\
     v_{t,i} &= k_t^{\top}s_t + \mu_t, \quad&  \mu_t \sim \mathcal{N}(0, r_t),
\end{aligned}
\end{equation}
where $\alpha_t \in [0, 1]$ is a learned decay factor controlling how much past information to retain. This corresponds to setting $A_t = \alpha_t I_n$ in \eqref{eq: LGM}. When $\alpha_t \to 0$, the state "forgets" the past completely; when $\alpha_t \to 1$, the state is fully retained.

Under the identity covariance assumption $\Sigma_t = I_n$, the Kalman gain becomes:
\begin{equation}\label{eq: kalman gain update GDN}
    G_t = \frac{(\alpha_t^2 + 1)k_t}{(\alpha_t^2 + 1)k_t^{\top}k_t + r_t} = \frac{k_t}{1 + r_t/(\alpha_t^2+1)},
\end{equation}
where the second equality again assumed unit-normalized keys (as in DeltaNet). Defining $\beta_t~=~(1 + \frac{r_t}{\alpha_t^2 + 1})^{-1}$ and substituting into the state update \eqref{eq:KF-recursive-for-LGM} yields:
\begin{equation}
\begin{aligned}
     \hat{s}_t &= \alpha_t \hat{s}_{t-1} + \beta_t k_t(v_{t,i} - k_t^{\top}\Big[\alpha_t \hat{s}_{t-1}\Big]),\\
     &= \Big[I_n - \beta_tk_tk_t^{\top}\Big]\alpha_t \hat{s}_{t-1} + \beta_t k_tv_{t,i},
\end{aligned}\tag{GDN}
\end{equation}
which recovers the GDN recurrence. In practice, $\beta_t$ is an input-dependent learnable parameter.

Like DeltaNet, GDN avoids tracking the evolving uncertainty $\Sigma_t$, trading optimality for computational simplicity. The key difference is that GDN's explicit forgetting factor $\alpha_t$ provides additional control over the memory horizon. However, by fixing $\Sigma_t = I_n$, GDN still ignores how measurement history should optimally influence the Kalman gain, leading to suboptimal performance compared to \ourshortname (see \cref{section: experiments}).

\noindent\textbf{Kimi Delta Attention (KDA)} \cite{team2025kimi} further extends GDN by using channel-specific decay factors $\alpha_{t,i}$ in place of the global $\alpha_t$. This allows different channels to have independent memory horizons. In the KF framework, this corresponds to:
\begin{equation}
s_{t,i} = \alpha_{t,i} s_{t-1,i} + w_{t,i} \quad w_{t,i} \sim \mathcal{N}(0, I_n),
\end{equation}
for each channel $i$. While this added flexibility can improve expressiveness, KDA still assumes $\Sigma_t = I_n$ and therefore does not optimally consider the entire past when computing its state update. Like DeltaNet and GDN, KDA sacrifices optimality for computational simplicity.

\section{Discussions and Limitations}
Thanks to its expressive test-time ridge regression objective, \ourname extends previous fading memory layers like Mamba2, LongHorn and Gated DeltaNet, all of which only depend on an instantaneous test-time objective. 
However, \ourshortname is only optimal among linear memory layers, solving our test-time objective leveraging non-linear updates while still maintaining hardware efficiency and numerical stability is an interesting area for future research. Despite the efficient kernels we implemented, we believe even faster implementations of our idea are possible, e.g., via \textit{sketching} (see \cref{appendix: sketching} for preliminary results). Finally, while we have showed promising results in combining \ourshortname with Attention layers into Hybrid models (\cref{appendix: hybrid}), further scaling beyond 3B parameters models is required to validate \ourshortname on more challenging real world problems.

\newpage
{
    \small
    \bibliographystyle{ieeenat_fullname}
    \bibliography{Liangzu}
}

\appendix

\onecolumn

\section{Related Work}\label{appendix:related work}
Since the introduction of Self-Attention \cite{Vaswani-NeurIPS2017}, significant research has been conducted to reduce its quadratic cost in processing long input sequences.
As models and systems scale to million-token contexts, Attention's bottlenecks have become critical blockers to frontier agentic applications in coding, information gathering, and scientific discovery \cite{chen2024scienceagentbench, cui2025curie, jimenez2023swe}.
Prior works have proposed various approximation schemes to overcome these limitations. For example, Reformer \cite{kitaev2020reformer} uses locality-sensitive hashing to group tokens with similar embeddings. This enables the model to attend only to a subset of tokens rather than the entire sequence. Other works equip Transformer models with "compressed" memory tokens that are updated dynamically and causally over sliding windows on entire sequence chunks \cite{dai2019transformerxl, munkhdalai2024leave, mohtashami2023landmark}.
While a lot of prior work have focused on reducing the quadratic complexity of Attention with sparse approximations \cite{nunez2024expansion, yuan2025NSA}, this work focuses on linear approximations of Attention.

\subsection{Linear Attention}
Linear Attention methods approximate the Attention mechanism with constant-size recurrent dynamical systems \cite{Dao-ICML2024-mamba2, Yang-ICML2024-gla, beck2024xlstm, Yang-ICLR2025}. Numerous State-Space Model (SSM) variations have been proposed, ranging from those closely resembling Linear Attention \cite{sun2023retentive} or Linear Time-Invariant dynamical systems \cite{gu2021combining, zancato2022stacked}, to those introducing novel adaptive or gated state updates \cite{Yang-ICML2024-gla, Dao-ICML2024-mamba2, orvieto2023resurrecting}.

Despite their differences, all SSMs follow the same basic working principle inspired by classical state-space models \cite{Kalman-1960}: they process the input sequence by maintaining a \textit{fixed-size} state that acts as a compressed (lossy) representation of all processed tokens. Moreover, when implemented in hardware, the state must have finite precision and ``fades away the past" as more samples are processed. Successful SSM layers typically employ hardware-aware implementations that efficiently utilize modern matrix multiplication accelerators through highly parallelizable and scalable primitives, including associative scans \cite{gu2023mamba, de2024griffin}, chunking mechanisms \cite{Dao-ICML2024-mamba2, Yang-ICML2024-gla}, and techniques that avoid materializing the entire state in slow high-bandwidth memory \cite{gu2023mamba}.

From a modeling perspective, most Linear Attention implementations introduce data-dependent gating factors to control the speed of their ``fading'' memory, balancing expressivity with scalability. For example, the transition from Mamba to Mamba2 replaced channel-wise data-dependent gating with head-wise gating for better scalability and Tensor Cores utilization. Input-dependent Gating has been shown to empirically improve training stability \cite{arora2023zoology, Yang-ICLR2025} and has driven the development of Linear Attention models (e.g., from S4 \cite{alber_gu_s4} to Mamba \cite{gu2023mamba} and from DeltaNet \cite{Yang-DeltaNet} to Gated DeltaNet \cite{Yang-ICLR2025}). In our work, we demonstrate that gating emerges naturally as a consequence of solving a weighted least squares objective function, establishing a connection to the favorable numerical properties classically described in the adaptive filtering literature \cite{LJUNG_RLS_stability, sayed2003fundamentals, sayed2011adaptive}.

\subsection{Hybrid State Space Attention Models}\label{section:appendix_hybrid_intro}
While extending the recurrent state in SSM layers has yielded performant models, they typically underperform on tasks requiring recall of information from the distant past \cite{waleffe2024empirical, jelassi2024repeat}. Hybrid State-Space Models address this limitation by complementing SSMs' ``fading" state with Attention layers \cite{dao2024transformers, de2024griffin, lieber2024jamba, glorioso2024zamba}. Early architectures simply stacked SSMs and Attention layers with different blending ratios \cite{waleffe2024empirical, gu2023mamba, Dao-ICML2024-mamba2} or replaced full Attention layers with Sliding Window Attention \cite{de2024griffin}. More sophisticated designs have recently emerged \cite{glorioso2024zamba, Zancato-NeurIPS2024-BMOJO}.

Notably, B'MOJO \cite{Zancato-NeurIPS2024-BMOJO} complements SSMs' fading state with "eidetic" memory by combining SSMs with Sliding Window Attention (SWA) in a single layer. Within the window, tokens can attend to a selected set of past tokens that were deemed difficult to predict using an asynchronous causal selection mechanism. B'MOJO was the first hybrid model to propose a parallel fusion of SSM and SWA at the layer level. Subsequent works \cite{dong2024hymba,bae2025hybrid} have shown this parallel fusion approach to be more performant (at equivalent compute) than the stacked approach of earlier works.


Thanks to their lower memory footprint and test-time scalability over long sequences, Hybrid architectures are expanding into long-range agentic tasks and have recently been trained with Reinforcement Learning at scale \cite{chen2025minimax}. When coupled with system-level optimizations like prefix caching \cite{pan2024marconi} and specialized inference engines \cite{kwon2023efficient}, Hybrid models can increase the number of rollouts (exploration), thereby improving end-to-end performance in Reinforcement Learning loops.

\newpage

\section{Forward and Backward Passes of  Chebyshev Iteration (Details)} \label{appendix: forward+backward CH}
In \cref{section:CH} we described our chunk-wise implementation of the CH method with adaptive regularization and gating. We now give full details omitted there.

\subsection{Forward Pass}
\label{appendix: forward CH}
\myparagraph{CH in Detail} We begin with describing the CH method (\cref{algo:Chebyshev}) in more detail. Assume we have a linear system of equations $H \xi = q$ where $H$ is a $D\times D$ positive definite matrix. We assume $H$ has its all eigenvalues lie in the interval $[\mu,L]$ and the values of $\mu$ and $L$ is known. Note that solving this system is equivalent to solving the following quadratic problem:
\begin{align}\label{eq:quadratic}
    \min_{\xi\in \bbR^D}\ \frac{1}{2} \xi^\top H \xi - \xi^\top q.
\end{align}
The classic Chebyshev Iteration in its standard form is presented in \cref{algo:Chebyshev}. In the initialization phase, we set $\rho=\frac{L-\mu}{L+\mu}$, which is the typical convergence rate of gradient descent applied to the above quadratic problem with stepsize $\frac{2}{L+\mu}$; vaguely speaking, in this setting, this stepsize choice is optimal (e.g., that allows gradient descent to converge the fastest possible). \cref{algo:Chebyshev} initializes two points, $\xi_{-1}$ and $\xi_0$. Here $\xi_{-1}$ is zero, and $\xi_0$ is a gradient step for \cref{eq:quadratic} starting at $\xi_{-1}$ and with stepsize $\frac{2}{L+\mu}$. The final component in initialization is the weight $\omega_0=2$. This is the starting point for the weight schedule recursion of $\omega_i$ in \cref{eq:weight-update}. Similarly, the initialization of $\xi_{-1}, \xi_0$ is where we start to compute $\xi_i$, whose update consists of \cref{eq:ch-gd} and \cref{eq:ch-momentum}. Note that \cref{eq:ch-gd} is with stepsize $2\cdot \omega_i /(L+\mu)$. Since $\omega_i>1$, this stepsize is strictly larger than $2/(L+\mu)$, the latter being the optimal stepsize for vanilla gradient descent. Such a large stepsize alone might not guarantee convergence, but it is balanced by the  \cref{eq:ch-momentum} term $\xi_{i-1}-\xi_{i-2}$ with positive weight $\omega_i-1$ so that the convergence of the Chebyshev iterative method is ensured.

\myparagraph{Numerical Stability Considerations} Now we analyze the numerical properties of the Chebyshev Iteration. The major computation consists of matrix-vector multiplication; in a batched parallel implementation, this turns out to be matrix-matrix multiplication. For this, the numerical accuracy is well controlled (e.g., in Triton we could specify the accuracy in \texttt{tl.dot}). The update of $\omega_i$ in \cref{eq:weight-update} might raise numerical concerns as it involves division. That said, we show this division operates in a numerically well-behaved range as $\omega_i$ is decreasing with $i$ yet lower bounded by $1$:

\begin{lemma}\label{lemma:omega-decrease}
    For any $r$, we have $2=\omega_0\geq \cdots \geq \omega_{r}\geq \omega^*_1>1$, where $\omega^*_1$ is defined as 
    \begin{align*}
        \omega^*_1 := \frac{ 2(1 - \sqrt{1-\rho^2})}{\rho^2}.
    \end{align*}
    As a consequence, we have $4-\rho^2\omega_i\in[2,4]$ for all $i=0,\dots,r$.
\end{lemma}
\begin{proof}
    If $L=\mu$, then $H$ is a scaled identity matrix, and the algorithm is simplified a lot. So we assume $L>\mu$ in what follows. 
    With $L> \mu>0$ we have  $\rho\in(0,1)$. Since $\omega_0=2$, we have $4-\rho^2\omega_0\geq 2$ and therefore $0<\omega_1\leq 2$. Repeating this argument and we see $\omega_i\in (0,2]$ for all $i$. By the definition of $\omega_i$, to show $\omega_i \leq \omega_{i-1}$ is to show
    \begin{align*}
        \frac{4}{4 - \rho^2 \omega_{i-1}} \leq \omega_{i-1} \Leftrightarrow g(\omega_i) \leq 0
    \end{align*}
    where $g$ is defined as $g(\omega) =\rho^2 \omega^2 - 4\omega + 4$. Note that $g(\omega)$ has two roots, $\omega_1^*$, as defined earlier, and $\omega_2^*=\frac{ 2(1 + \sqrt{1-\rho^2})}{\rho^2}$; $\omega_1^*,\omega_2^*$ are the two fixed points of the update \cref{eq:weight-update}. Observing that $\omega_0=2$ lies in the interval $(\omega_1^*, \omega_2^*)$, and moreover, for any $i\geq 1$, if $\omega_{i-1}> \omega_1^*$ we must have 
    \begin{align*}
        \omega_i = \frac{4}{4- \rho^2 \omega_{i-1}}  > \frac{4}{4- \rho^2 \omega_1^*} =\omega_1^*.
    \end{align*}
    This proves $\omega_i>\omega_1^*$ for all $i=1,\dots,r$. Next, since $\omega_0=2$ lies in the interval $(\omega_1^*, \omega_2^*)$ where $g(\omega)$ decreases, therefore we have $\omega_1\leq \omega_0$. Thus $\omega_1$ lies in $(\omega_1^*, \omega_2^*)$ again. We could then conclude inductively that $\omega_1^* < \omega_i\leq \omega_{i-1}$ for all $i=1,\dots,r$.
\end{proof}
From \cref{lemma:omega-decrease} we know that the update of $\omega_i$ in \cref{eq:weight-update} would not create much numerical concern in a forward pass, as we have $\omega_i\in[1,2]$ for all $i$. Furthermore, we can bound the rate at which $\omega_i$ converges to $\omega_1^*$:
\begin{lemma}\label{lemma:omega-rate}
    Define $\kappa:= \frac{L}{\mu}$. For any $i=1,\dots,r$, we have 
    \begin{align*}
        (\omega_i - \omega_1^*) \leq R^i \cdot (\omega_0 - \omega_1^*), 
    \end{align*}
    where $R$ is defined as 
    \begin{align*}
        R := \frac{\kappa-1}{\kappa+1} \cdot \frac{\sqrt{\kappa} -1 }{ \sqrt{\kappa}+1 }.
    \end{align*}
\end{lemma}
\begin{proof}
    From the update rule of $\omega_i$ in \cref{eq:weight-update} and the fixed point property of $\omega_1^*$, we have
    \begin{align*}
        \omega_i - \omega_1^* &= \frac{4}{4 - \rho^2 \omega_{i-1}} - \frac{4}{4 - \rho^2 \omega_1^*} \\ 
        &=\frac{4\rho^2 }{(4 - \rho^2 \omega_{i-1})(4 - \rho^2 \omega_1^*)} \cdot (\omega_{i-1} - \omega_1^*) \\ 
        &\overset{\text{(i)}}{=}  \frac{\rho^2 \omega_1^* }{4 - \rho^2 \omega_{i-1}} \cdot (\omega_{i-1} - \omega_1^*) \\ 
        &\overset{\text{(ii)}}{\leq} \frac{\rho^2 w_{i-1} \omega_1^* }{4} \cdot (\omega_{i-1} - \omega_1^*) \\
        &\overset{\text{(iii)}}{\leq} \left(1-\sqrt{1-\rho^2} \right) \cdot (\omega_{i-1} - \omega_1^*) \\ 
        &\overset{\text{(iv)}}{=} \left( \frac{\kappa-1}{\kappa+1} \cdot \frac{\sqrt{\kappa} -1 }{ \sqrt{\kappa}+1 } \right) \cdot (\omega_{i-1} - \omega_1^*)
    \end{align*}
Here, (i) follows from the fact that $\omega_1^*$ is a fixed point, (ii) follows from \cref{lemma:omega-decrease} that $\omega_i \leq \omega_{i-1}$, (iii) follows from the definition of $\omega_1^*$ and the fact $w_{i-1}\leq 2$, and (iv) follows from the definitions of $\kappa$ and $\rho$. The proof is concluded by unrolling the above recurrence.
\end{proof}
\begin{remark}\label{remark:omega_rate}
    Here, we call $R$ the linear convergence rate (or \textit{contraction factor}) of $\omega_i$ to $\omega_1^*$. First-order methods for solving $H\xi = q$ converge at most at a rate $R_a:=\frac{\sqrt{\kappa} -1 }{ \sqrt{\kappa}+1 }$, and we see $\omega_i$ converges at an even faster rate. Numerically, assuming $\kappa=\frac{L}{\mu}=\frac{1.02}{0.02}=51$, we then have:
    \begin{align*}
        R&\approx 0.7253,\quad R^5 \approx 0.2, \quad R^{10} \approx 0.04,\quad  R^{20} \approx 0.0016, \quad R^{30} \approx 6\times 10^{-5} \\ 
        R_a&\approx 0.7543, \quad  R_a^5 \approx 0.244, \quad R_a^{10} \approx 0.0597, \quad R_a^{20} \approx 0.0036, \quad R_a^{30} \approx 0.0002.
    \end{align*}
    Thus, with $\kappa=51$, the update of $\omega_i$ in \cref{eq:weight-update} converges in at most 20 iterations up to the bfloat16 precision.    
\end{remark}


\subsection{Backward Pass}
\label{appendix: backward CH}
We now give details for backpropagation through the Chebyshev Iteration (\cref{algo:Chebyshev}) via implicit differentiation. 

\myparagraph{Computing  $\frac{dL}{dq_t}$ and $\frac{dL}{dk_t}$} First, we follow \cref{table:implicit-diff} and \cref{lemma:CH-autoq=CH-implq}, and compute $dq_t$. Then, given the equation $(H_t + \lambda_t I )dq_t = d x_t^* $, we have that 
\begin{align}
    d(H_t + \lambda_t I) = - dq_t  (x_t^*)^\top. 
\end{align}
Therefore $d\lambda_t = \text{tr}( d(H_t + \lambda_t I) ) = -(x_t^*)^\top  dq_t $. Since we set $\lambda_t=  a\cdot \| H_t \|_{\text{F}}$, this indicates
\begin{align}\label{eq:dH_t-implicit}
    d H_t = -dq_t  (x_t^*)^\top - a \cdot \frac{H_t}{\| H_t \|_{\text{F}} } \cdot \left((x_t^*)^\top  dq_t \right).
\end{align}
Note that this expression of $dH_t$ is \textit{partial}: It accounts for the upstream gradient from $dq_t$ only and one might think of the subsequent states all depend on $H_t$. We will accumulate the gradients later when needed.

Now, the recursion of $H_t$ in \cref{eq:ch-forward} implies
\begin{align}\label{eq:dk_i}
    d k_i &=  \sum_{t\geq i}  \left( d H_t  + (d H_t)^\top \right) k_i \cdot \frac{\zeta_t}{\zeta_i} \\ 
    &=\sum_{t\geq i}  \frac{\zeta_t}{\zeta_i}  \left( -dq_t \otimes (x_t^*)^\top - x_t^* \otimes (dq_t)^\top + w_t H_t  \right) k_i, 
\end{align}
which proves \cref{lemma:dk_i}. We refer the reader to \cref{subsubsection:dqdk} for more detailed derivations of $dq_t$ and $d k_t$.

\myparagraph{Derivatives for Gating} In practice we often parameterize $\gamma_t$ in the log space to ensure numerical stability. Thus, let us first revise our notations for this case. Let $g_t=\log \gamma_t $ and $G_t:=\sum_{i=1}^t g_i = \log\left( \prod_{i=1}^t \gamma_i \right)$. Then the mask matrix $M$ is 
\begin{align}
    M_{i,j} = \begin{cases}
        \exp( G_{j}- G_i) & j\geq i; \\ 
        0 & \text{otherwise}.
    \end{cases}
\end{align}
Now, since for any $c=1,\dots,C$ we have
\begin{align}
    H_c = \exp(G_c) \cdot H_{0} + \sum_{j=1}^c \exp(G_c- G_j)  \cdot  k_{j} k_{j}^\top, 
\end{align}
for any $G_i$ we have the following basic derivatives:
\begin{align}
    \frac{d H_c}{d G_i} &= \begin{cases}
        0  & c < i; \\ 
        \exp(G_i) \cdot H_{0} + \sum_{j=1}^{i-1} \exp(G_i- G_j)  \cdot  k_{j} k_{j}^\top & c = i; \\ 
       - \exp (G_c - G_i) k_i k_i^\top & c > i,
    \end{cases} \quad \quad i=1,\dots, C, \quad c = 1,\dots, C; \\ 
    \frac{d H_{C+1}}{d G_i} &=  - \exp (G_{C+1} - G_i) k_i k_i^\top \quad \quad \quad   i=1,\dots, C
\end{align}
With $dH_{C+1}$ being the aggregated gradient from the future, we  have for $i=1,\dots,C$ that
\begin{align}
    dG_i &=   \sum_{c = i }^{C+1} \langle dH_{c}, \frac{d H_c}{ d G_i } \rangle \\
     &= e^{G_i} \langle dH_{i}, H_0 \rangle +   \sum_{j=1}^{i-1} e^{G_i- G_j}  \langle dH_i, k_{j} k_{j}^\top \rangle - \sum_{c=i+1}^C e^{G_c - G_i} \langle dH_c, k_{i} k_{i}^\top \rangle - e^{G_{C+1} - G_i}  \langle dH_{C+1}, k_{i} k_{i}^\top \rangle \\ 
     &= e^{G_i} \langle dH_{i}, H_0 \rangle +   \sum_{j=1}^{i} e^{G_i- G_j}  \langle dH_i, k_{j} k_{j}^\top \rangle - \sum_{c=i}^C e^{G_c - G_i} \langle dH_c, k_{i} k_{i}^\top \rangle - e^{G_{C+1} - G_i}  \langle dH_{C+1}, k_{i} k_{i}^\top \rangle \\ 
     dG_{C+1} &= e^{G_{C+1}} \langle dH_{C+1}, H_0 \rangle +   \sum_{j=1}^{C} e^{G_{C+1}- G_j}  \langle dH_{C+1}, k_{j} k_{j}^\top \rangle 
\end{align}
Note that in one of the above equations we add and subtract the term $\langle dH_i, k_ik_i^\top \rangle$, which will simplify the implementation.

Recall that $d H_t = -dq_t  (x_t^*)^\top - \frac{1}{2}\cdot w_t H_t$ with $w_t=\frac{ 2 a  \cdot (x_t^*)^\top  dq_t}{\| H_t \|_{\text{F}} }$. In computing the derivatives of $G_i$ the first term $dq_t  (x_t^*)^\top$ is the standard term that arises in that of \cref{eq:linear-attn}, which we omit here. We now focus on the second term $\frac{1}{2}\cdot w_t H_t$. This implies the gradients $dG_i$ and $dG_{C+1}$ are partly given respectively by (using the notations in \cref{eq:CH-backward-mainpaper} and omitting some algebraic operations)
\begin{align}
    \frac{1}{2}\cdot \langle w_i H_i, H_i \rangle - \frac{1}{2}\cdot k_i^\top B_i k_i \quad \quad \text{and} \quad \quad \frac{1}{2}\cdot e^{G_C} \langle \widetilde{B}_{C+1}, H_0 \rangle +  \frac{1}{2}\cdot \sum_{j=1}^C e^{G_C - G_j} \cdot k_j^\top  \widetilde{B}_{C+1}k_j.
\end{align}
Computing the first term $\langle w_i H_i, H_i \rangle$ in parallel is easy by invoking the definition of $w_i$ and the Frobenius norm of $H_i$ we stored during the forward pass. Computing the quadratic terms $k_i^\top B_i k_i$ and $ k_j^\top  \widetilde{B}_{C+1}k_j$ in parallel is easy and follows from our computation of $B_ik_i$ and $\widetilde{B}_{C+1}k_i$ for $dk_i$ in \cref{eq:CH-backward-mainpaper}. Computing $\langle \widetilde{B}_{C+1}, H_0 \rangle$ is easy since we recompute the initial states $H_0$ of each chunk and have them available during the backward pass, while $\widetilde{B}_{C+1}$ is updated backwards in a for loop.

\subsubsection{Computing  $\frac{dL}{dq_t}$ and $\frac{dL}{dk_t}$.}\label{subsubsection:dqdk}

In forward pass we solve \[(H_t + \lambda_tI)x_t = q_t\]
\begin{equation}
    \begin{aligned}
        x_t &= (H_t + \lambda_tI)^{-1}q_t \\
        \implies dx_t &= \underset{J_{q_t \to x_t}}{\underbrace{(H_t + \lambda_tI)^{-1}}}dq_t
    \end{aligned}
\end{equation}
Recall that the gradient is transpose of the Jacobian, thus we obtain
\begin{equation}
    \label{eq. implicity dl/dq}
    \frac{dL}{dq_t} = (H_t + \lambda_tI)^{-1}\frac{dL}{dx_t}.
\end{equation}
Thus, we can obtain $\frac{dL}{dq_t}$ by running a Chebyshev iteration to solve (for $z$) the linear system of equations
\[(H_t + \lambda_tI)z = \frac{dL}{dx_t}.\]
Now we have
\begin{equation}
\begin{aligned}
\label{eq: dH}
    dx_t &= d(H_t + \lambda_tI)^{-1}q_t \\
    dx_t &= -(H_t + \lambda_tI)^{-1}d(H_t + \lambda_tI)(H_t + \lambda_tI)^{-1}q_t \\
    dx_t &= -(H_t + \lambda_tI)^{-1}d(H_t + \lambda_tI)x_t \\
    &= (x_t^\top  \otimes -(H_t + \lambda_tI)^{-1}) \textrm{vec}(d(H_t + \lambda_tI))
    \end{aligned}
\end{equation}
In the last equality we have used the identity $\textrm{vec}(ABC) = (C^\top  \otimes A)\textrm{vec}(B)$.

Now we will compute the Jacobian of $\lambda$ with respect to $H_t$:
\begin{equation}
\label{eq: dlambda}
    \begin{aligned}
        \lambda_t &= a||H_t||_F \\
        &= a \sqrt{\textrm{Tr}(H_t^\top H_t)} \\
    \implies d\lambda &= a d\Big(\sqrt{\textrm{Tr}(H_t^\top H_t)}\Big) \\
    &= a \frac{1}{2||H_t||_F}\textrm{Tr}((dH_t)^\top H_t + H_t^\top  (dH_t)) \\
    &= a \frac{1}{||H_t||_F}\textrm{vec}(H_t)^\top d\textrm{vec}(H_t)\\
    \end{aligned}
\end{equation}

Substituting \eqref{eq: dlambda} in \eqref{eq: dH}.
\begin{equation}
\begin{aligned}
\label{eq: dH_full}
    dx_t 
    &= (x_t^\top  \otimes -(H_t + \lambda_tI)^{-1}) \Big(\textrm{vec}(dH_t) + \frac{a}{||H_t||_F}\textrm{vec}(I)\textrm{vec}(H_t)^\top \textrm{vec}(dH_t))\Big) \\
    \end{aligned}
\end{equation}

Thus, we can obtain $\textrm{vec}(\frac{dL}{dH_t})$ as,
\begin{equation}
    \textrm{vec}(\frac{dL}{dH_t}) = (x_t \otimes -(H_t + \lambda_tI)^{-1})\frac{dL}{dx_t} +\frac{a}{||H_t||_F}\textrm{vec}(H_t)\textrm{vec}(I)^\top (x_t \otimes -(H_t + \lambda_tI)^{-1})\frac{dL}{dx_t}
\end{equation}
Substituting from \eqref{eq. implicity dl/dq}, 
\begin{equation}
\begin{aligned}
    \textrm{vec}(\frac{dL}{dH_t}) &= -(x_t \otimes \frac{dL}{dq_t}) - \frac{a}{||H_t||_F}\textrm{vec}(H_t)\textrm{vec}(I)^\top (x_t \otimes \frac{dL}{dq_t}) \\
    &= -(x_t \otimes \frac{dL}{dq_t}) - \frac{a}{||H_t||_F}\textrm{vec}(H_t)\langle \frac{dL}{dq_t}, x_t \rangle
\end{aligned}
\end{equation}

Now, with gating, we have $H_t = \gamma_t H_{t - 1} + k_tk_t^\top $. Which can be unrolled as 
\begin{equation}
    \label{eq: cumulative dynamics}
    H_l = \sum_{i=0}^l \Big(\prod_{k=i}^{l-1} \gamma_k \Big) k_i k_i^\top 
\end{equation}

We will compute $\frac{dL}{dk_l}$ for some $l \leq t$, 
\begin{equation}
\label{eq: dk equation}
    \begin{aligned}
        \frac{dL}{dk_l} = \sum_{t \geq l}\frac{d\textrm{vec}(H_t)}{dk_l}\textrm{vec}(\frac{dL}{dH_t})
    \end{aligned}
\end{equation}
Computing $\frac{d\textrm{vec}(H_t)}{dk_l}$ for some $t \geq l$

\begin{equation}
    H_t = \prod_{i=l}^{t-1}\gamma_i k_l k_l^\top  + \textrm{terms indep. of $k_l$.}
\end{equation}
Taking differentials on both sides,
\begin{equation}
\begin{aligned}
\label{eq: Jacobian H wrt k}
    dH_t &= \prod_{i=l}^{t-1}\gamma_i \Big[dk_l k_l^\top  + k_l (dk_l)^\top \Big] \\
    d(\textrm{vec}(H_t)) &=  \prod_{i=l}^{t-1}\gamma_i \Big[\textrm{vec}(dk_l k_l^\top )  + \textrm{vec}(k_l (dk_l)^\top )\Big] \\
    &= \underset{J_{k_l \to \textrm{vec}(H)}}{\underbrace{\Big(\prod_{i=l}^{t-1}\gamma_i \Big)\Big[(k_l \otimes I)  + (I \otimes k_l) \Big]}}dk_l
\end{aligned}
\end{equation}
where in the last equality we used the identity $\textrm{vec}(k_l dk_l^\top ) = dk_l \otimes k_l = (I \otimes k_l)dk_l $.

Subsituting the Jacobian (transposed for gradients) from \eqref{eq: Jacobian H wrt k} to \eqref{eq: dk equation} we obtain.

\begin{equation}
\label{eq: dk equation final}
    \begin{aligned}
        \frac{dL}{dk_l} = \sum_{t \geq l}\Big(\prod_{i=l}^{t-1}\gamma_i \Big)\Big[(k_l^\top  \otimes I)  + (I \otimes k_l^\top ) \Big]\textrm{vec}(\frac{dL}{dH_t})
    \end{aligned}
\end{equation}

Substituting the expression for $\text{vec}(\frac{dL}{dH_t})$ into equation \eqref{eq: dk equation final} we get:

\begin{equation}
    \begin{aligned}
        \frac{dL}{dk_l} & = -\sum_{t \geq l}\left(\prod_{i=l}^{t-1}\gamma_i \right) &
        \Big[(k_l^\top  \otimes I)  + (I \otimes k_l^\top )\Big]
        \Big[(x_t \otimes \frac{dL}{dq_t}) + \frac{a}{||H_t||_F}\textrm{vec}(H_t)\langle \frac{dL}{dq_t}, x_t \rangle \Big] = \\
         & = -\sum_{t \geq l}\left(\prod_{i=l}^{t-1}\gamma_i \right) & \Big[ \Big( 
         (k_l^\top  \otimes I)(x_t \otimes \frac{dL}{dq_t}) + 
         \frac{a}{||H_t||_F} (k_l^\top  \otimes I) \textrm{vec}(H_t)\langle \frac{dL}{dq_t}, x_t \rangle \Big) + \\
         & & 
         (I \otimes k_l^\top )(x_t \otimes \frac{dL}{dq_t}) + 
         \frac{a}{||H_t||_F} (I \otimes k_l^\top )  \textrm{vec}(H_t)\langle \frac{dL}{dq_t}, x_t \rangle \Big]
    \end{aligned}
\end{equation}

Note that the following equations hold: 
\begin{equation}
    \begin{aligned}
        (k_l^\top  \otimes I)(x_t \otimes \frac{dL}{dq_t}) & = (k_l^\top  x_t \otimes \frac{dL}{dq_t}) = \langle k_l, x_t \rangle \frac{dL}{dq_t} \\ 
        (I \otimes k_l^\top )(x_t \otimes \frac{dL}{dq_t}) & = x_t \otimes k_l^\top  \frac{dL}{dq_t} = \langle k_l, \frac{dL}{dq_t} \rangle x_t
    \end{aligned}
\end{equation}
since $(A \otimes B)(C \otimes D) = AC \otimes BD$ and the fact that the Kronecker products after the simplification is a scalar times a vector. 

For the other terms is holds:
\begin{equation}
    \begin{aligned}
         (k_l^\top  \otimes I) \textrm{vec}(H_t) & = \textrm{vec}(H_t k_l) \\
         (I \otimes k_l^\top )  \textrm{vec}(H_t) & = \textrm{vec}(k_l^\top  H_t) = \textrm{vec}(H_t^\top  k_l) 
    \end{aligned}
\end{equation}
where we used the fact $\textrm{vec}(AXB) =  (B^\top  \otimes A)\textrm{vec}(X)$ and the fact that the $\textrm{vec}$ operator applied to a row vector returns the same result as applying it on its transpose (so we go from $k_l^\top  H_t$ to $H_t^\top  k_l$). Since $H_t$ is symmetric we can sum both contributions and get twice that amount. 

Eventually we get: 

\begin{equation}
\boxed{
\begin{aligned}
\frac{dL}{dk_l} = -\sum_{t \geq l}\left(\prod_{i=l}^{t-1}\gamma_i \right)\Bigg[&\langle k_l, x_t \rangle \frac{dL}{dq_t} + \langle k_l, \frac{dL}{dq_t} \rangle x_t + \frac{2a}{||H_t||_F}\langle x_t, \frac{dL}{dq_t} \rangle H_t k_l\Bigg]
\end{aligned}
}
\end{equation}

Or equivalently, collecting the terms that are linear in the gradient:

\begin{equation}
\boxed{
\frac{dL}{dk_l} = -\sum_{t \geq l}\left(\prod_{i=l}^{t-1}\gamma_i \right)\left[\langle k_l, x_t \rangle \frac{dL}{dq_t} + \langle k_l, \frac{dL}{dq_t} \rangle x_t + \frac{2a\langle x_t, \frac{dL}{dq_t} \rangle}{||H_t||_F} H_t k_l\right]
}
\end{equation}
Note: The last term creates a dependence on $k_l$ through $H_t k_l$, which is expected since the regularization $\lambda_t$ couples the gradient computation.

\newpage

\section{Proof of \cref{lemma:CH-autoq=CH-implq}}
\label{appendix:proof-implicit-explicit-dq}
\cref{lemma:CH-autoq=CH-implq} describes an interesting phenomenon where, for the CH method (\cref{algo:Chebyshev}), the gradient $d\hat{q}$ obtained from implicit differentiation coincides with the exact gradient $dq$ obtained via backpropagation (chain rule). To prove this result, one way is to derive an analytic expression for $dq$ (\cref{subsection:CH-BP-details}) and then inspect the recursions. However, this can be algebraically involved. Here, we present a clear proof based on some simple observations.

First, note that the output $\xi_r$ is linear in $q$ and moreover there is a matrix function $p_r(H)\in \bbR^{D\times D}$ such that
\begin{align}
    \xi_r = p_r(H) \cdot q.
\end{align}
Here $p_r(H)$ is a polynomial function of $H$ that encodes the Chebyshev iteration (\cref{algo:Chebyshev}). Conversely,  we understand that $p_r(H) \cdot q$ can be computed by applying the Chebyshev iteration with $H, q$ for $r$ iterations (together with other parameters such as $\mu,L$). Then, given the output gradient $d\xi_r$, we have
\begin{align}
    dq = p_r(H)^\top \cdot  d\xi_r = p_r(H) \cdot d\xi_r,  
\end{align}
where the last equality follows, since $H$ is symmetric, which implies $p_r(H)$ is symmetric. The proof is finished by observing that $p_r(H) \cdot d\xi_r$ can be computed via \cref{algo:Chebyshev} with $H=H, q=d\xi_r$ and other parameters, which gives us $dq$.

\subsection{The Exact Backward Pass for $dq$ and $dH$}\label{subsection:CH-BP-details}

Here we show how to obtain the exact gradients of $dH$ and $dq$ in \cref{algo:Chebyshev} given the output gradient $d\xi_r$, which might be of independent interests. The key insight here is that the Chebyshev iteration can be \textit{reversed}.

\myparagraph{Backward Pass for dq} Let $I$ be the identity matrix of suitable size. To derive a backward pass of \cref{algo:Chebyshev}, we first write down the update of $\xi_i$ concisely in the following recursion
\begin{align}\label{eq:3recursion}
    \xi_i = A_i \xi_{i-1} + b_i \xi_{i-2} + c_i q,
\end{align}
where $A_i,b_i,c_i$ are defined as 
\begin{align}\label{eq:Abc}
    A_i = \omega_i I - \frac{2\cdot \omega_i}{ L+\mu } H, \quad b_i = -(\omega_i-1), \quad c_i = \frac{2\cdot \omega_i}{ L+\mu }.
\end{align}
Note that $A_i$ is symmetric. Define $d \xi_i:= \frac{d \cL }{d \xi_i}$ for every $i$. With some loss function $\cL$, assume we are now given $d \xi_r$, and our goal is to compute $d q := \frac{d \cL }{d q}$. Since $q$ appears in \cref{eq:3recursion} for every $i$, we know $\xi_0, \xi_1,\dots,\xi_r$ all depend on $q$. Therefore, with $c_0:= \frac{2}{L+\mu}$, we have
\begin{align*}
    dq =  \sum_{i=0}^r  c_i \cdot  d \xi_i.
\end{align*}
It remains to compute $d \xi_i$ for every $i$. Applying the chain rule to \cref{eq:3recursion}, we obtain 
\begin{equation}\label{eq:dxi}
    \begin{split}
        d \xi_{r-1} &= A_r \cdot d\xi_r \\  
    d \xi_{i-2} &= A_{i-1} \cdot d\xi_{i-1} + b_i  \cdot d \xi_i, \quad \forall i=r,\dots, 2.
    \end{split}
\end{equation}
Note that $A_i,b_i,c_i$ depend on some constant terms and $\omega_i$. Thus, to compute them backward we assume access to $\omega_r$ and these constants. By reversing \cref{eq:weight-update} we derive the following recursion:
\begin{equation}\label{eq:weight-update-backward}
    \begin{split}
        \nu_r  &\gets \omega_r\\
    \nu_{i-1} &\gets \frac{4}{\rho^2} \left( 1- \frac{1}{\nu_i} \right), \quad \forall i=r,\dots, 1.
    \end{split}
\end{equation}
Similarly to how $\omega_i$ decreases with $i$ and converges to $\omega_1^*$, we may prove $\nu_i$ is convergent to the other fixed point, $\omega_2^*$, as $i$ decreases (and the iterate does not stop at $i=1$).

\myparagraph{Backward Pass for dA} From \cref{eq:3recursion} and \cref{eq:Abc} we see that 
\begin{align}\label{eq:dA}
    dA_i = d \xi_i \otimes \xi_{i-1}^\top, \quad d H = - \frac{2\cdot \omega_i}{L+\mu} \cdot  d A_i
\end{align}
where $\otimes$ denotes the Kronecker product; this is the out product of $d \xi_i$ and $\xi_{i-1}^\top$, as $d \xi_i$ and $\xi_{i-1}$ are vectors.

\myparagraph{Reverse Chebyshev Iteration} At first glance, computing $dA_i=d \xi_i \otimes \xi_{i-1}^\top$ requires storing $\xi_{i-1}$ in the forward pass, and the actual calculation of $dA_i$ is done after we run the backward pass for $d \xi_i$ in \cref{eq:dxi}. However, storing all $\xi_i$'s would be memory-inefficient. To address this issue, a main insight here is that we can reverse  \cref{eq:3recursion} and write
\begin{align}\label{eq:xi_backward}
   \xi_{i-2} = \frac{1}{b_i}  ( \xi_i -  A_i \xi_{i-1} + c_i q).
\end{align}
This implies that we can recover all the iterates $\xi_{r},\dots,\xi_0$ as soon as we have access to the last two, $\xi_r,\xi_{r-1}$. Therefore, to obtain $dA_i$, we can run two iteration schemes in \cref{eq:dxi} and \cref{eq:xi_backward} simultaneously.

\begin{remark}
    We find that being able to run the iterative update \textit{backward} in a numerically stable fashion is a main feature of the Chebyshev iterative method (or more generally, gradient descent variants with momentum). Vanilla gradient descent can not efficiently reverse its iterate $\xi_i = \xi_{i-1} - \gamma_i (H \xi_{i-1} - q)$ with stepsize $\gamma_i$, as it requires inverting $(I-\gamma_i H)$. Moreover, reversing \cref{eq:xi_backward} can be done stably, as $b_i$ is often in a good numerical range, which means division by $b_i$ in \cref{eq:xi_backward} is not an issue. To see this, first note that by \cref{lemma:omega-decrease} we have 
    \begin{align*}
        1\geq -b_i=\omega_i - 1 \geq \omega_1^* - 1.
    \end{align*}
    Note that $\omega_1^*$ defined in \cref{lemma:omega-decrease} is an increasing function of $\rho$ and therefore of $\kappa$. We then have that $-b_i\in [0.25, 1]$ for any $\kappa \geq 10$ (we will not consider the case $\kappa<10$ as this means we need to add a very large regularization strength which might harm the minimization of the regression loss). In comparison, if we were to reverse the CG iteration, we would need to divide a quantity that is often numerically as small as $10^{-3}$ or as large as $10^{10}$ (see \cref{fig:reverse_division}). This is why it is numerically unstable to reverse CG.
\end{remark}

\begin{figure}
    \centering
    \includegraphics[width=0.49\linewidth]{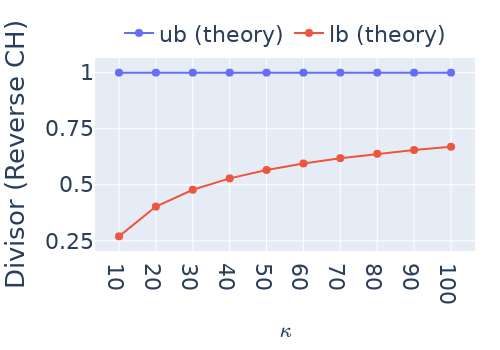}
    \includegraphics[width=0.49\linewidth]{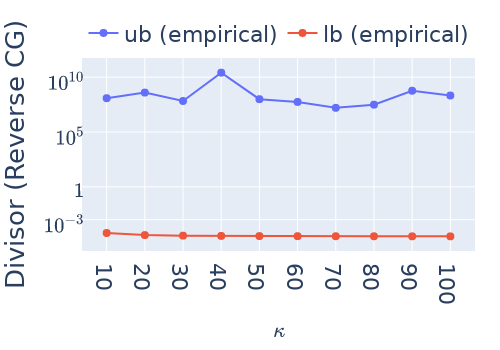}

    \makebox[0.49\textwidth]{\footnotesize \quad \quad \ \ (a)  }
    \makebox[0.49\textwidth]{\footnotesize \quad \quad (b) }
    
    \caption{(a) The theoretical lower and upper bounds for the values of the divisor $b_i$ that arise in reversing Chebyshev \cref{eq:xi_backward}; (b) The empirical lower and upper bounds for the divisor that arises in reversing CG. }
    \label{fig:reverse_division}
\end{figure}

\begin{algorithm}[t]
    \SetAlgoLined
    \DontPrintSemicolon

    \texttt{Input}: $H, d \xi_{r}$, $L,\mu$, , number of iterations $r$, the final weight $\omega_r$; 

    Initialize $\rho\gets \frac{L-\mu}{L+\mu}, d \xi_{r+1}\gets0$, $\nu_r \gets \omega_r, \nu_{r+1}\gets0, \nu_0\gets 1$ , $dq\gets \frac{2\nu_{r}}{L+\mu} \cdot d\xi_r$, $dH\gets -  \frac{2\nu_{r}}{L+\mu} d \xi_r \otimes \xi_{r-1}^\top$;

    For $i=r,\dots,1$: 
    \begin{align}
        \xi_{i-2} & \gets -\frac{1}{\nu_i-1}  \left( \xi_i -  \left(\nu_i I - \frac{2\cdot \nu_i}{ L+\mu } H \right) \xi_{i-1} + \frac{2\cdot \nu_i}{ L+\mu } q \right) \\ 
        d \xi_{i-1} &\gets \left( \nu_i \cdot d \xi_i - \frac{2\cdot \nu_i}{ L+\mu } \cdot H  \cdot d\xi_{i}  \right)-(\nu_{i+1}-1)  \cdot d \xi_{i+1} \\ 
        \nu_{i-1} & \gets \frac{4}{\rho^2} \left( 1- \frac{1}{\nu_i} \right) \\ 
        d q &\gets d q +  \frac{ 2\nu_{i-1} }{L+\mu} \cdot d \xi_{i-1} \\ 
        dH &\gets dH - \frac{2\nu_{i-1}}{L+\mu} \cdot d \xi_{i-1} \otimes \xi_{i-2}^\top \label{eq:dH-1}
    \end{align}
    \texttt{Output}: $d q$, $dH$;
    
    \caption{Backward Pass of Chebyshev Iteration } 
    \label{algo:Chebyshev-backward}
\end{algorithm}

\section{Experimental Setup}
\label{sec: Model architecture}
\myparagraph{Model Configurations} We consider models of 3 different sizes: 440M, 1B, and 2.8B. This is summarized in \cref{table:model_size_arch}. All models are with the GPT2 tokenizer similarly to \cite{Von-arXiv2025-mesa}.

\begin{table}[]
    \centering
    \caption{Model sizes and the corresponding architectural configurations.}
    \begin{tabular}{cccc}
     \toprule
     Model Size & Number of Layers & Number of Heads & Hidden Dimension  \\  
     \midrule 
      440M & 28 & 8 & 1024  \\
      1B   & 28 & 12 & 1536 \\ 
      2.8B & 32 & 20 & 2560 \\
      \bottomrule
    \end{tabular}
    \label{table:model_size_arch}
\end{table}

\myparagraph{Training Configurations} 
All models are trained with the AdamW optimizer with initial learning rate $10^{-3}$, $5\%$ warm-up steps, cosine schedule, gradient clipping with maximum norm $1$. 


\begin{table}[]
    \centering
    \caption{Model sizes and the corresponding architectural configurations.}
    \begin{tabular}{cccc}
     \toprule
     Model Size & Global Batch Size & Total Number of Training Tokens & Sequence Length\\  
     \midrule 
      440M & 1M & 8B  & 2048  \\
      1B   & 2M & 20B & 2048 \\ 
      2.8B & 2M & 100B & 4096 \\
      \bottomrule
    \end{tabular}
    \label{table:model_size_training}
\end{table}

Models of the same scale use the same training configurations. Specifically (see also \cref{table:model_size_training}):
\begin{itemize}
    \item For 440M models, we use sequence length 2048 and 8B DCLM tokens.
    \item For 1B models, we use sequence length 2048 and 20B DCLM tokens.
    \item For 2.8B models, we use sequence length 4096 and 100B DCLM tokens.
\end{itemize} 

\myparagraph{Model Hyperparameters} We use default parameters for all other models as given in the  \href{https://github.com/fla-org/flash-linear-attention}{Flash-Linear-Attention v0.4.0} library (except the ones mentioned in \cref{table:model_size_arch}). For our approach, we use $\lambda_t = 0.02 \cdot \| H_t \|_{\text{F}}$, with gating and $\alpha$-connection enabled by default, unless otherwise specified. We also run CH for $30$ iterations for all experiments.

\myparagraph{Individual Experiments} We now describe the setups for each individual experiment. 

In \cref{fig:cgch-grad}a, we randomly generate tensors $k\in \bbR^{B\times T\times H\times D}$ and $q$ and normalize them along the last dimension ($D$). Here $B,T, H,D$ simulate the batch size, sequence length, number of heads, and head dimension, respectively. Then we compute the covariance matrices $H\in \bbR^{B\times T\times H\times D\times D}$ of $k$, normalize its every $D\times D$ slice by its Frobenius norm. The code to generate data is shown below.

\begin{verbatim}
    k = torch.randn(B, T, H, D).to(dtype).to('cuda')
    q = torch.randn(B, T, H, D).to(dtype).to('cuda')

    q = q / torch.linalg.vector_norm(q, dim=-1, keepdim=True).to(q)
    k = k / torch.linalg.vector_norm(k, dim=-1, keepdim=True).to(k)

    kk = torch.einsum('...i,...j->...ij', k, k).cumsum(1)  
    kk = kk / torch.linalg.matrix_norm(kk, ord='fro')[..., None, None]

    kk.diagonal(dim1=-2, dim2=-1).add_(ridge_strength)
\end{verbatim}
For \cref{fig:cgch-grad}c and \cref{fig:cgch-grad}e, we generate random input ids with vocabulary size 5000, sequence length 2048 within a 5-layer LLAMA; we set 2 heads and head dimension 128 for this architecture.

In the MQAR experiments of \cref{fig:mqar+fig:speed}a, we follow the standard experimental setting but consider a strictly harder setting with smaller model dimension (or hidden dimension). Indeed, in the setting of \cite{arora2023zoology}, the model dimension is always larger than or equal to the number of KV pairs, while in the setting here, in some cases the model dimension is smaller than the number of KV pairs, in which case linear SSMs could not perfectly memorize all KV pairs.

In the main paper, \cref{fig:mqar+fig:speed}a is without any gating or $\alpha$-connection.



In \cref{fig:longctx_ssm} we considered the following tasks for long context evaluations. Reported results for each task is average over the score obtained for individual datasets in that task.
\begin{itemize}
    \item \textit{Retrieval-Augmented Generation (RAG)}: These tasks consist of open-domain question answer where the model is given a gold passage (passage containing the answer) interspersed between many other retrieved passages from a corpus \cite[Wikipedia dump split into 100-word passages]{petroni2021kilt}. The model is tasked with answering the question based on the obtained passages. We consider the following datasets from HELMET \cite{yen2025helmet} for this task: Natural Questions, TriviaQA, PopQA, HotpotQA.
    \item \textit{Many-shot In-Context Learning (ICL)}: ICL tests LLMs ability to learn new skills from a few examples. Here the task is to learn to classify between different concepts based on several in-context examples of the said concept. We consider the following datasets from HELMET \cite{yen2025helmet} for this task: TREC Coarse, TREC Fine, NLU, BANKING77, CLINIC150.
    \item \textit{Synthetic Recall}: These tasks are variations of the ``Needle-in-a-Haystack" task \cite{needle2024haystack} where the goal is to retrieve an important piece of information, the ``needle" from a long context of distractor tokens, the ``haystack". These variations also test multi-hop tracing and aggregation capabilities of the model.
    We consider the following datasets from RULER \cite{hsieh2024ruler} for this task: S-NIAH-{1/2/3}, MK-NIAH-{1,2,3}, MV-NIAH, MQ-NIAH, VT, CWE, FWE.
    \item \textit{LongQA}: These are long document based question-answering tasks. The documents are typically made long by randomly sampling different paragraphs from the same dataset along with the paragraph that contains the answer. We consider the following datasets from RULER \cite{hsieh2024ruler} for this task: SQuAD, HotpotQA.
\end{itemize}

\newpage
\section{Synthetic Recall Without S-NIAH-1}\label{section:no-S-NIAH}
As noted in the main paper, synthetic recall tasks differ significantly from natural text. For instance, in \textit{S-NIAH-1}, 19 of 20 words can repeat 10,000+ times while ground-truth appears once. GKA's regression objective then overweights irrelevant tokens, hindering retrieval. Excluding this S-NIAH-1 task, GKA is competitive on synthetic recall (10 remaining RULER tasks) with better average (over RAG, ICL, Recall \& LongQA) performance compared to other SSMs, as shown in \cref{fig:Recall_GKA}.

\begin{figure}[t]
    \centering
    \includegraphics[width=0.7\textwidth]{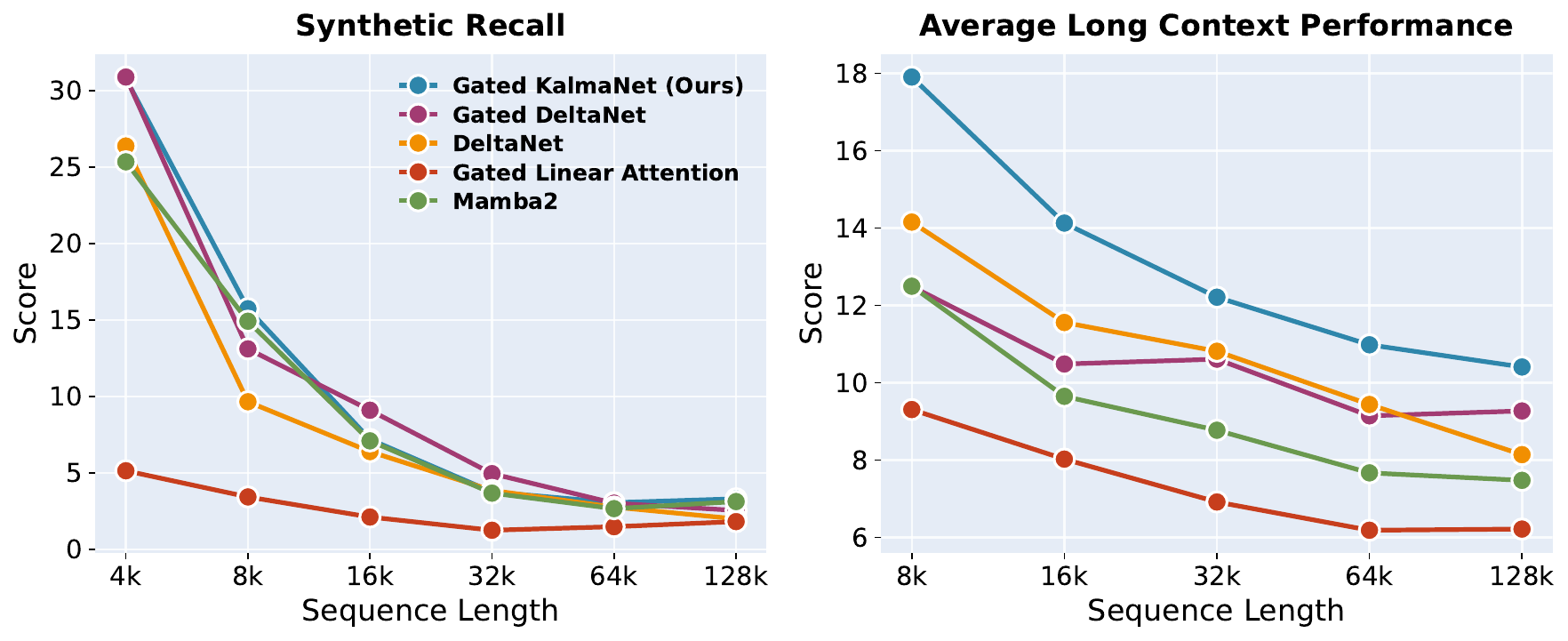}
    \vspace{-0.1cm}
    \caption{Recall and average performance without \textit{S-NIAH-1}.}
    \label{fig:Recall_GKA}
\end{figure}

\section{Throughput for Long Sequences}\label{section:throughput-long-seq}
\cref{table:throughput} profiles GKA's throughput across model scales and context lengths. Due to its linear-time complexity, GKA maintains nearly constant throughput as sequence length increases, given a fixed batch size of 4M tokens.

\begin{table}[htbp]
\centering
\caption{Throughput in tokens/GPU/sec. GKA achieves up to 2$\times$ higher throughput than Transformers at 64K context length.}
\label{table:throughput}
\small
\resizebox{0.5\textwidth}{!}{
\begin{tabular}{l|l|ccc}
\toprule
\textbf{Scale} & \textbf{Model} & \textbf{16K} & \textbf{32K} & \textbf{64K} \\
\midrule
\midrule
8B & Transformer & 6096.37 & 4766.25 & 3382.50 \\
& GKA & 6898.53 & 6808.94 & 6721.64 \\
\midrule
14B & Transformer & 2833.99 & 2221.56 & 1485.24 \\
& GKA & 3318.28 & 3256.45 & 3256.45 \\
\bottomrule
\end{tabular}
}
\end{table}

\newpage

\section{How Does The Performance \ourshortname Scale with Compute?}\label{appendix: GKA scaling laws}
We consider models at three different scales: 440M, 1B and 2.8B. For training configurations and architecture refer to \cref{sec: Model architecture}. We use prototypical tasks from LM-Harness (see \cref{subsubsection: short-context tasks} for list of tasks) to evaluate language modeling capabilities of \ourshortname and compare with baseline SSM/fading memory layers. \cref{table:ssm_lm_harness_comparison_model_params} shows that at 440M scale, \ourshortname is competitive with GDN and Deltanet. However, differences emerge at larger scales, with \ourshortname showing increasing benefits. In particular, the retrieval capabilities of our model, as measured by FDA and SWDE consistently outperform all SSM baselines at 1B and 2.8B scale. We also report the results of equal-sized Transformer for completeness, which serves as a performance ceiling at each scale.

\begin{table}[htbp]
\centering
\caption{\textbf{\ourshortname shows stronger scaling with compute that other SSM baseline models.} LM-Harness results for models at different scales: 440M, 1B and 2.8B. All models were trained from scratch. 440M and 1B models were trained on 8B and 20B tokens respectively in accordance to the Chinchila scaling laws \cite{hoffmann2022empirical}. For the 2.8B model we trained on 100B tokens.}
\label{table:ssm_lm_harness_comparison_model_params}
\small
\resizebox{0.9\textwidth}{!}{
\begin{tabular}{l|cccccccccc|c}
\toprule
\textbf{Model} & \textbf{\begin{tabular}{@{}c@{}}ARC-C\end{tabular}} & \textbf{\begin{tabular}{@{}c@{}}ARC-E\end{tabular}} & \textbf{\begin{tabular}{@{}c@{}}BoolQ\end{tabular}} & \textbf{\begin{tabular}{@{}c@{}}COPA\end{tabular}} & \textbf{\begin{tabular}{@{}c@{}}HellaSWAG\end{tabular}} & \textbf{\begin{tabular}{@{}c@{}}PIQA\end{tabular}} & \textbf{\begin{tabular}{@{}c@{}}SciQ\end{tabular}} & \textbf{\begin{tabular}{@{}c@{}}Winogrande\end{tabular}} & \textbf{\begin{tabular}{@{}c@{}}FDA\end{tabular}} & \textbf{\begin{tabular}{@{}c@{}}SWDE\end{tabular}} & \textbf{\begin{tabular}{@{}c@{}}Avg\end{tabular}} \\
& acc\_n $\uparrow$ & acc\_n  $\uparrow$ & acc $\uparrow$ & acc $\uparrow$ & acc\_n $\uparrow$ & acc\_n $\uparrow$ & acc\_n $\uparrow$ & acc $\uparrow$ & contains $\uparrow$ & contains $\uparrow$ & \\
\midrule
\midrule
\multicolumn{12}{c}{\textit{440M Models}} \\
\midrule
Transformer & 24.40 & \underline{42.26} & \underline{59.88} & 70.00 & 36.19 & 64.15 & 61.50 & \textbf{51.70} & \textbf{5.17} & \textbf{35.64} & \textbf{45.09} \\
Gated Linear Attention & 24.06 & 40.28 & 56.57 & \underline{71.00} & 32.70 & 62.24 & 57.80 & 50.67 & 1.00 & 9.18 & 40.55 \\
Gated DeltaNet & \textbf{25.17} & 41.96 & 58.23 & \textbf{72.00} & 36.96 & \textbf{64.69} & \underline{63.6} & \textbf{51.7} & 1.91 & 11.88 & \underline{42.81} \\
DeltaNet & \underline{25.09} & 41.92 & \textbf{61.13} & 65.00 & \underline{37.20} & \underline{64.47} & \textbf{64.00} & 49.49 & \underline{2.81} & \underline{14.31} & 42.54 \\
Gated KalmaNet (Ours) & 24.57 & \textbf{43.22} & 56.94 & \underline{71.00} & \textbf{37.22} & \underline{64.47} & 62.8 & 50.83 & 1.45 & 14.04 & 42.65 \\
\midrule
\multicolumn{12}{c}{\textit{1B Models}} \\
\midrule
Transformer & 26.62 & 46.42 & 59.94 & \textbf{77.00} & 44.01 & 67.14 & \textbf{68.30} & 54.06 & \textbf{8.35} & \textbf{45.18} & \textbf{49.70} \\
Mamba2 & \textbf{28.07} & \underline{46.63} & \underline{60.21} & 70.00 & \underline{44.57} & 67.57 & 65.50 & \underline{54.30} & 1.45 & 15.75 & 45.40 \\
Gated Linear Attention & 25.94 & 42.00 & 58.84 & 70.00 & 36.34 & 63.60 & 58.20 & 51.85 & 1.45 & 10.53 & 41.88 \\
Gated DeltaNet & 27.05 & \textbf{47.98} & 59.54 & \underline{74.00} & 44.27 & 67.36 & 66.2 & 53.83 & 2.18 & 17.82 & 46.02 \\
DeltaNet & \underline{27.56} & 46.25 & 59.97 & 71.00 & 43.18 & \underline{67.74} & 65.90 & \textbf{55.41} & 3.09 & 20.61 & 46.07 \\
Gated KalmaNet (Ours) & 25.43 & 46.55 & \textbf{60.73} & \underline{74.00} & \textbf{44.59} & \textbf{68.88} & \underline{67.60} & 52.41 & \underline{6.17} & \underline{21.87} & \underline{46.82} \\
\midrule
\multicolumn{12}{c}{\textit{2.8B Models}} \\
\midrule
Transformer & 32.25 & 56.10 & \textbf{64.28} & 80.00 & 60.96 & 73.56 & 79.50 & 61.72 & \textbf{58.53} & \textbf{72.28} & \textbf{63.92} \\
Mamba2 & 32.24 & 59.64 & 58.72 & \underline{82.00} & 62.23 & 73.78 & 79.80 & 62.19 & 7.71 & 41.13 & 55.94 \\
Gated Linear Attention & 27.82 & 50.80 & 52.57 & 78.00 & 48.83 & 70.13 & 69.60 & 54.54 & 2.81 & 20.43 & 47.55 \\
Gated DeltaNet & \underline{32.59} & \textbf{60.02} & \underline{62.75} & \underline{82.00} & \underline{62.8} & \underline{74.32} & \underline{80.6} & \underline{62.35} & 8.26 & 44.28 & 57.00 \\
DeltaNet & \textbf{32.85} & 58.16 & 42.51 & 81.00 & 61.13 & 73.78 & 43.90 & 61.72 & 11.80 & 46.08 & 51.29 \\
Gated KalmaNet (Ours) & 32.51 & \underline{59.89} & 61.68 & \textbf{85.00} & \textbf{63.84} & \textbf{74.81} & \textbf{83.2} & \textbf{64.17} & \underline{12.89} & \underline{50.95} & \underline{58.89} \\
\bottomrule
\end{tabular}
}
\end{table}

\newpage

\section{Ablations}
\label{appendix: ablations}
In this section we consider ablations for various modeling choices made in arriving at our final \ourshortname model. For all ablations, we consider 2.8B models trained on 100B tokens on DCLM at 4K context length (unless mentioned otherwise). We use the same architecture and training configurations for these ablations as mentioned in \cref{sec: Model architecture}. 

\subsection{Does Adaptive Regularization Help?} 
As discussed in \cref{subsection:CH-adaptive-reg}, we introduced adaptive regularization to control the condition number of $H_T + \lambda_tI$ for numerical stability. Here we ablate this choice, specifically we compare the following runs.
\begin{enumerate}
    \item \textit{Adaptive regularization}. We train a model with $\lambda_t = a||H_t||_F$. We report results for $a=0.02$ for this run.
    \item \textit{Constant regularization} We train same model architecture (as above) with $\lambda_t = 0.25$ (a constant). This choice of $0.25$ is motivated from concurrent work \cite{Von-arXiv2025-mesa} which explored a similar ridge regression objective for LLM training. 
\end{enumerate}

As shown in \cref{fig:adaptive reg. ablation}, without strict condition number control, gradient norms spike during training, leading to increased cross entropy loss (compared to the run with adaptive regularization).

\begin{figure}[h]
    \centering
    \includegraphics[width=0.48\textwidth]{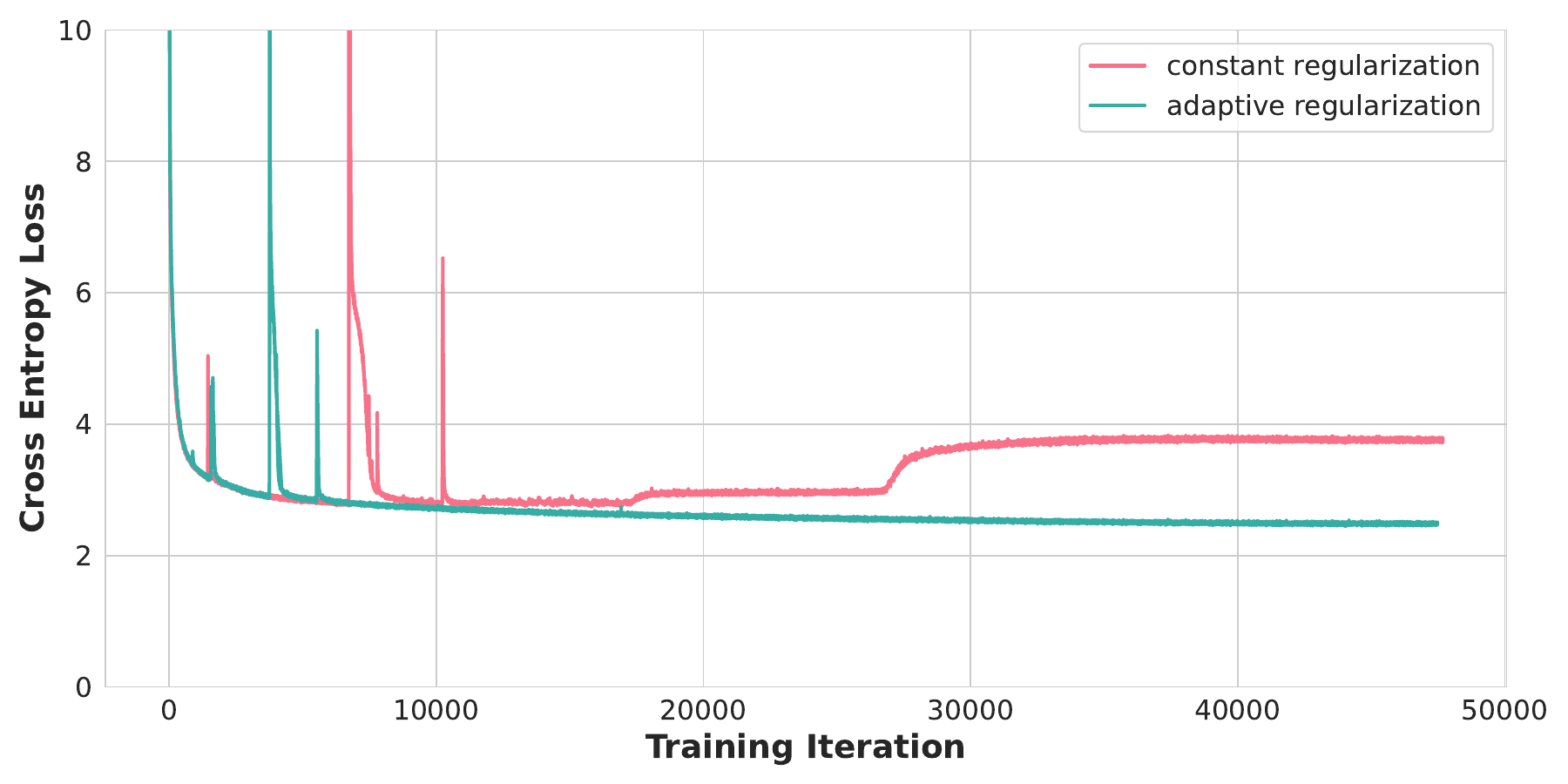}
    \includegraphics[width=0.48\textwidth]{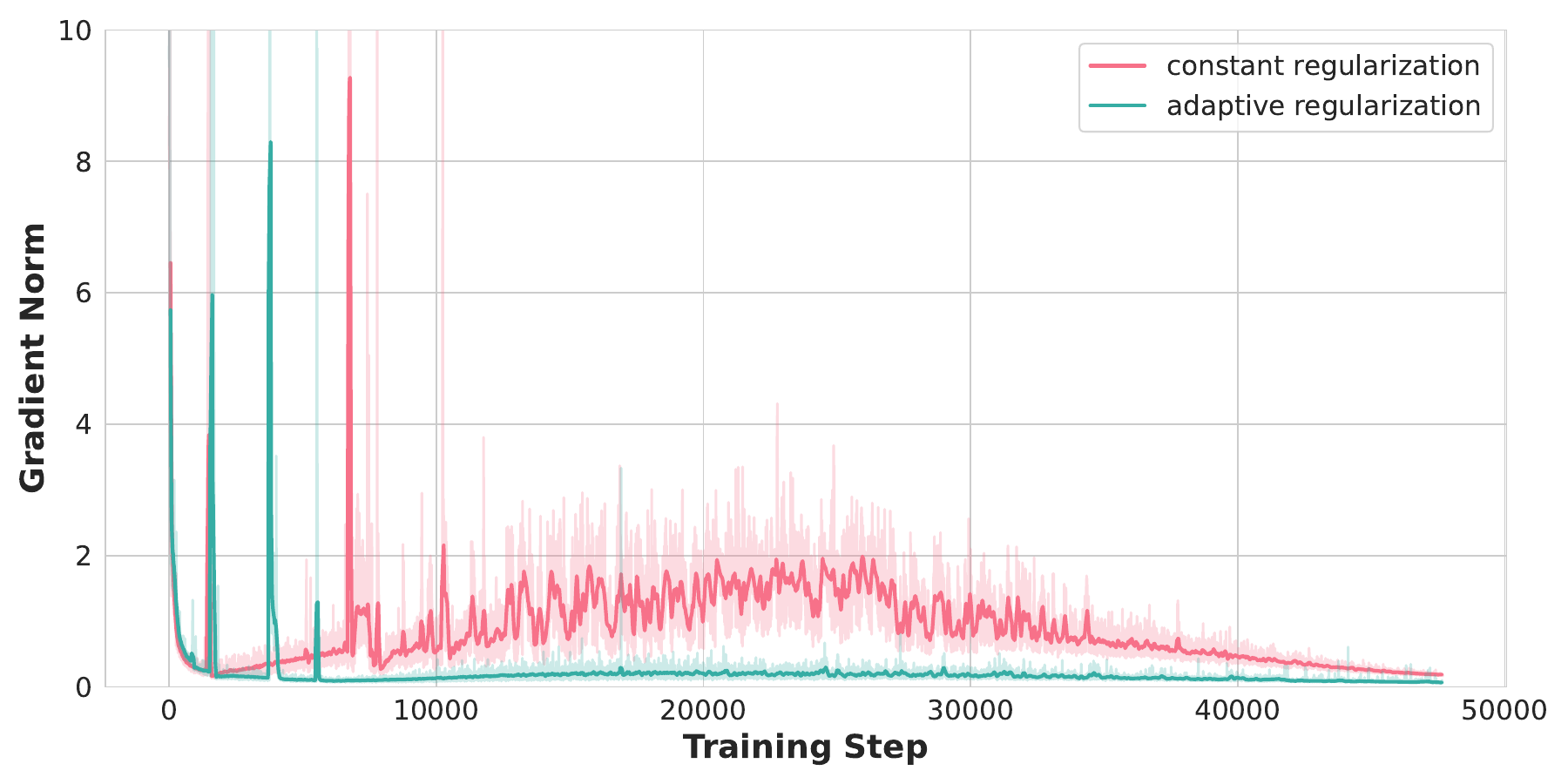}
    
    \caption{\textbf{Adaptive regularization results in smoother and better training curves.} (a) Plots the training curve for 2.8B models on 100B tokens from DCLM. (b) Plots the corresponding gradient norm. The model with constant regularization (red curve) results in a higher loss that can be attributed to its non-smooth trajectory over the course of its training run (spiky gradient norms).}
    \label{fig:adaptive reg. ablation}
    \vspace{-0.3cm}
\end{figure}

\subsection{Does Adaptive Weighting Help?} 
In \cref{subsection:CH-adaptive-reg}, we discussed increasing the expressivity of our layer by introducing adaptive weights $\eta_{t,i}$ which re-weigh the past to be exponentially decaying in time. Given constant-sized memory, we hypothesize this adaptive weighting (gating) allows \ourshortname to learn an effective representation by incorporating recency bias into its computation. In this subsection we test this hypothesis. We carry out the following runs.
\begin{enumerate}
    \item \textit{Adaptive weighting (gating)}. We train a model with adaptive weights. Specifically, for all $t\geq i$, we parameterize the weight for the $i^{\textrm{th}}$ sample at time-step $t$ as $\eta_{t,i} = \prod_{j = i+1}^t \gamma_j$, with each $\gamma_j \in [0,1]$ learnable.

    \item \textit{No weighting}. We train the same model architecture as above, but with no weights. This essentially results in an unweighted ridge regression objective obtained by setting $\eta_{i} = 1$ in \cref{eq:rls}. 
\end{enumerate}

\cref{table:gating_ablation} shows clear benefits of adapting weighting with improvements across the board on all LM-Harness tasks considered, thereby validating our hypothesis.

\begin{table*}[htbp]
\centering
\caption{\textbf{Adaptive weighting outperforms across the board on LM-Harness tasks.} Results for 2.8B models trained on 100B tokens from DCLM with and without adaptive weights as introduced in \cref{subsection:CH-adaptive-reg}.}
\label{table:gating_ablation}
\small
\resizebox{0.9\textwidth}{!}{
\begin{tabular}{c|cccccccccc|c}
\toprule
\textbf{Adaptive Weights} & \textbf{\begin{tabular}{@{}c@{}}ARC-C\end{tabular}} & \textbf{\begin{tabular}{@{}c@{}}ARC-E\end{tabular}} & \textbf{\begin{tabular}{@{}c@{}}BoolQ\end{tabular}} & \textbf{\begin{tabular}{@{}c@{}}COPA\end{tabular}} & \textbf{\begin{tabular}{@{}c@{}}HellaSWAG\end{tabular}} & \textbf{\begin{tabular}{@{}c@{}}PIQA\end{tabular}} & \textbf{\begin{tabular}{@{}c@{}}SciQ\end{tabular}} & \textbf{\begin{tabular}{@{}c@{}}Winogrande\end{tabular}} & \textbf{\begin{tabular}{@{}c@{}}FDA\end{tabular}} & \textbf{\begin{tabular}{@{}c@{}}SWDE\end{tabular}} & \textbf{\begin{tabular}{@{}c@{}}Avg\end{tabular}} \\
& acc\_n $\uparrow$ & acc\_n  $\uparrow$ & acc $\uparrow$ & acc $\uparrow$ & acc\_n $\uparrow$ & acc\_n $\uparrow$ & acc\_n $\uparrow$ & acc $\uparrow$ & contains $\uparrow$ & contains $\uparrow$ & \\
\midrule
\midrule
\xmark & 28.24 & 51.73 & 57.68 & 76 & 53.87 & 71.87 & 71.6 & 54.38 & 6.08 & 33.03 & 50.45 \\
\cmark & \textbf{32.51} & \textbf{59.89} & \textbf{61.68} & \textbf{85} & \textbf{63.84} & \textbf{74.81} & \textbf{83.2} & \textbf{64.17} & \textbf{12.89} & \textbf{50.95} & \textbf{58.89} \\
\bottomrule
\end{tabular}
}
\end{table*}

\subsection{Does $\alpha$-connection Improve Training of \ourshortname?} \label{appendix: alpha connection ablation}
In \cref{subsection:GKA-block}, we introduce the $\alpha$-connection as a residual connection that establishes a direct path for gradient flow through the GLA solution, improving training stability. This allows the model to fall back on the GLA solution when CH produces poor-quality results due to non-convergence of the iterative solver within the fixed iteration budget. To validate this design choice, we perform two runs.

\begin{enumerate}
    \item[R1.] \textit{with $\alpha$-connection}. We train a model with the $\alpha$-connection as shown in our GKA block in \cref{fig:GKA-block}.

    \item[R2.] \textit{without $\alpha$-connection}. We train the same model architecture as above, but with no $\alpha$ connection. This can be simply understood as setting $\alpha_t = 1$ for all time-steps $t$ in \cref{fig:GKA-block}. 
\end{enumerate}

On LM-Harness, both models perform similarly, with R1 and R2 achieving aggregate scores of 58.89 and 58.39, respectively. However, clear differences emerge under long-context evaluation, where we trained both models on an additional 25B tokens from long documents at 128K context length. \cref{fig:alpha_ablation} shows that \ourshortname without the $\alpha$-connection exhibits inferior long-context performance on average, with Synthetic Recall and LongQA showing major degradation.

\begin{figure}[h]
    \centering
    \includegraphics[width=1.0\textwidth]{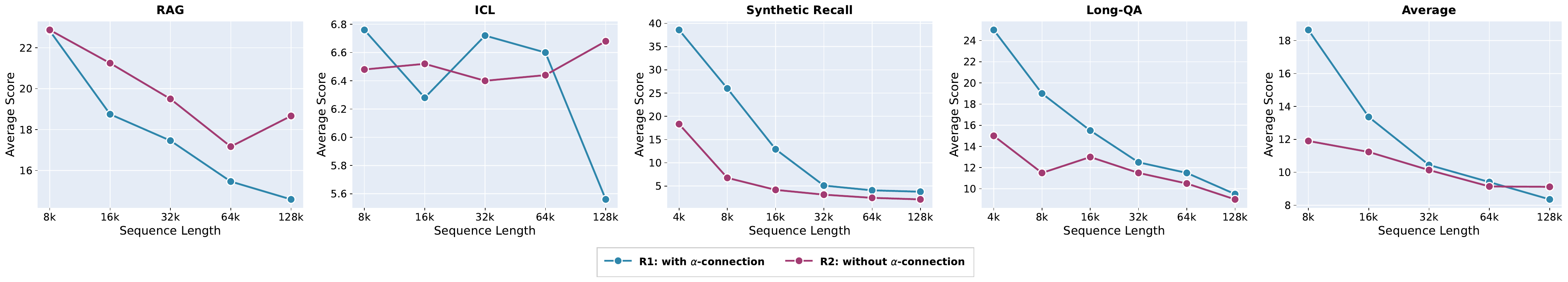}
    \caption{\textbf{\ourshortname without the $\alpha$- connection severaly underperforms on Synthetic Recall and LongQA.} On ICL all SSMs struggle to perform better than random chance (see \cref{fig:longctx_ssm}). Interestingly, although R2 exhibits poorer long-context abilities in aggregate, it outperforms R1 on RAG by a few points.}
    \label{fig:alpha_ablation}
    \vspace{-6pt}
\end{figure}

\newpage

\newpage
\section{Effects of Different Regularization Strengths}
\label{appendix: ablation-reg-strength}
Recall that we proposed setting adaptive regularization $\lambda_t=a \cdot \|H_t \|_{\text{F}}$. We now present experiments validating this choice. 

\myparagraph{Synthetic Experiments} First, we generate data as per \cref{fig:cgch-grad}a, where the covariance matrix is normalized by its Frobenius norm. In this case we set $\lambda_t = a$ for $a$ varying in $\{0.01, 0.02, 0.05, 0.1\}$. \cref{fig:convergence-ch} shows that the \textit{maximum regularized residual norm} (computed as the maximum of $\| (H_t+ \lambda_t I) \xi_i - q \|_2$ over all dimensions where $\xi_{i}$ is the estimate of CH at iteration $i$) decreases as we enlarge $\lambda_t$. This is because having a large $\lambda_t$ reduces the condition number. The downside, though, with a large $\lambda_t$ is that it reduces the memorization capacity, namely, it might enlarge $\| H_t \xi_i - q \|_2$, the true residual of interest.
\begin{figure}[h]
    \centering
    \includegraphics[width=0.24\textwidth]{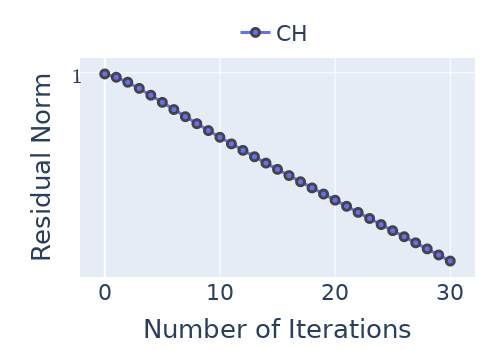}
    \includegraphics[width=0.24\textwidth]{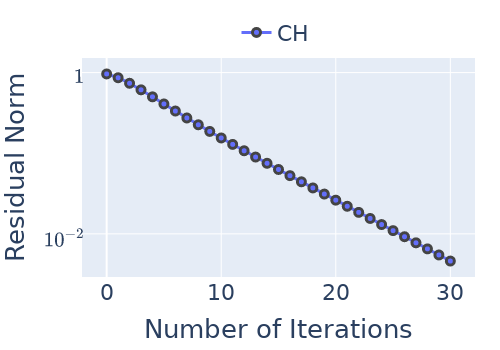}
    \includegraphics[width=0.24\textwidth]{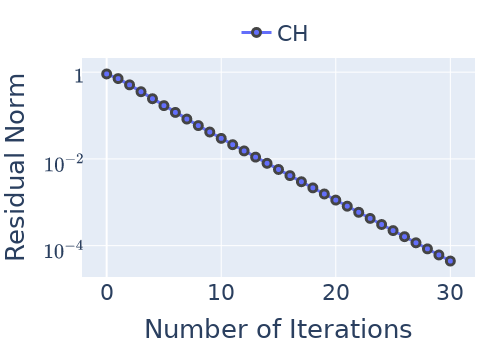}
    \includegraphics[width=0.24\textwidth]{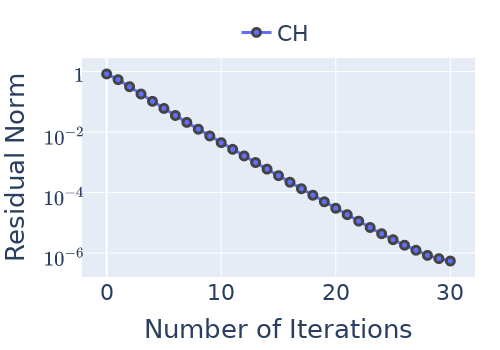}

    \makebox[0.24\textwidth]{\footnotesize \quad \quad  (a) $\lambda_t=0.01$}
    \makebox[0.24\textwidth]{\footnotesize \quad \quad  (b) $\lambda_t=0.02$}
    \makebox[0.24\textwidth]{\footnotesize \quad \quad  (c) $\lambda_t=0.05$}
    \makebox[0.24\textwidth]{\footnotesize \quad \quad  (d) $\lambda_t=0.1$} 


    
    \caption{Convergence for varying regularization strengths (batch size $8$, sequence length 2048, 8 heads, and head dimension 128). }
    \label{fig:convergence-ch}
\end{figure}

\myparagraph{\ourshortname with different regularization strengths} We train several 2.8B models with varying regularization strength by choosing $a \in [0.01, 0.02, 0.05, 0.1]$. While performance on LM-Harness (\cref{table:ablation_over_a}) shows little discrepancy, we observe noticeable differences in long-context performance—where memorization capacity matters most—(\cref{fig:lambda_ablation}). Specifically, the long-context performance of \ourshortname improves initially as we decrease $a$ from $0.1 \to 0.05$. This is expected since this increases the memorization capacity of the model. However, decreasing further from $0.05 \to 0.02 \to 0.01$ causes performance to decrease. This can be attributed to the increasing condition number of the problem, which reduces the quality of the solution computed by CH (\cref{fig:convergence-ch}).

\begin{figure}[h]
    \centering
    \includegraphics[width=1.0\textwidth]{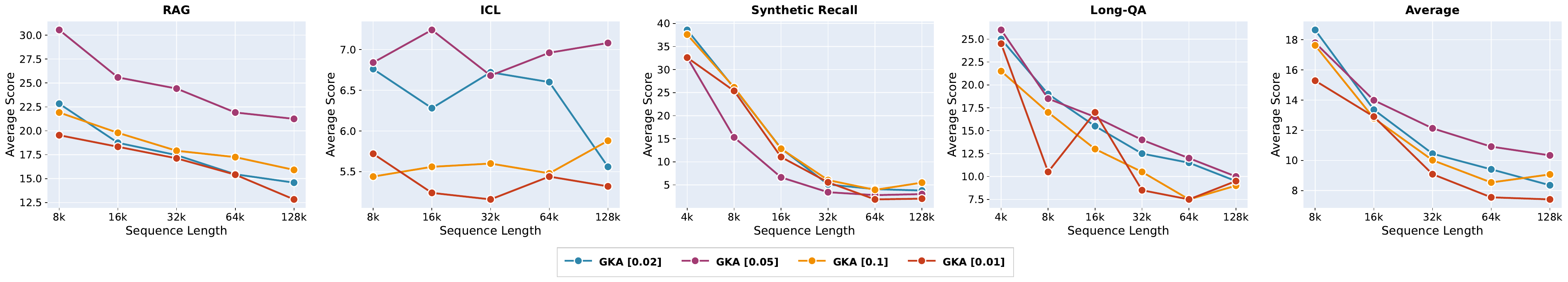}
    \caption{\textbf{Long context performance \ourshortname for different regularization strengths.} The long-context performance of \ourshortname improves initially as we decrease $a$ from $0.1 \to 0.05$. This is expected since this increases the memorization capacity of the model. However, decreasing further from $0.05 \to 0.02 \to 0.01$ causes performance to decrease. This can be attributed to the increasing condition number of the problem, which reduces the quality of the solution computed by CH (\cref{fig:convergence-ch})}
    \label{fig:lambda_ablation}
\end{figure}

\begin{table*}[htbp]
\centering
\caption{\textbf{Ablation over different choices of regularization strength $\lambda_t = a\cdot \| H_t\|_{\text{F}}$.} Short-context performance on LM-Harness shows little discrepancy with different regularization strengths.}
\label{table:ablation_over_a}
\small
\resizebox{0.9\textwidth}{!}{
\begin{tabular}{l|cccccccccc|c}
\toprule
\textbf{$a$} & \textbf{\begin{tabular}{@{}c@{}}ARC-C\end{tabular}} & \textbf{\begin{tabular}{@{}c@{}}ARC-E\end{tabular}} & \textbf{\begin{tabular}{@{}c@{}}BoolQ\end{tabular}} & \textbf{\begin{tabular}{@{}c@{}}COPA\end{tabular}} & \textbf{\begin{tabular}{@{}c@{}}HellaSWAG\end{tabular}} & \textbf{\begin{tabular}{@{}c@{}}PIQA\end{tabular}} & \textbf{\begin{tabular}{@{}c@{}}SciQ\end{tabular}} & \textbf{\begin{tabular}{@{}c@{}}Winogrande\end{tabular}} & \textbf{\begin{tabular}{@{}c@{}}FDA\end{tabular}} & \textbf{\begin{tabular}{@{}c@{}}SWDE\end{tabular}} & \textbf{\begin{tabular}{@{}c@{}}Avg\end{tabular}} \\
& acc\_n $\uparrow$ & acc\_n  $\uparrow$ & acc $\uparrow$ & acc $\uparrow$ & acc\_n $\uparrow$ & acc\_n $\uparrow$ & acc\_n $\uparrow$ & acc $\uparrow$ & contains $\uparrow$ & contains $\uparrow$ & \\
\midrule
\midrule
0.01 & \textbf{33.45} & 58.63 & 62.63 & \textbf{85.00} & 63.36 & 73.99 & 81.40 & 63.14 & 11.16 & \textbf{51.49} & 58.43 \\
0.02 & 32.51 & 59.89 & 61.68 & \textbf{85.00} & 63.84 & 74.81 & \textbf{83.20} & \textbf{64.17} & \textbf{12.89} & 50.95 & \textbf{58.89} \\
0.05 & 32.68 & \textbf{61.66} & 53.57 & 79.00 & 63.46 & \textbf{74.84} & 82.60 & 63.77 & 11.98 & 49.68 & 57.32 \\
0.1 & 32.76 & 59.85 & \textbf{63.52} & 84.00 & \textbf{63.95} & \textbf{75.08} & \textbf{83.20} & 63.54 & 11.43 & 51.22 & 58.86 \\
\bottomrule
\end{tabular}
}
\end{table*}

\newpage

\begin{table}[]
\caption{
\textbf{Latent sketching increases training throughput (by up to 10\%) while marginally reducing accuracy (< 1\%).}
Training throughput is reported in \# Billion tokens/day/node. It is measured on a single H200 GPU with a batch size of 1M tokens. Our results indicate minimal regression on LM-harness tasks but up to 10\% improvement in training throughput (going from no-sketch to sketch dim 32). However, long context performance is adversely affected with sketching with up to 60\% relative drop in performance. Future work will address this by exploring the use of sketching adaptively depending on the "complexity" of the task.}
    \centering
    \scalebox{0.9}{
    \begin{tabular}{lccc}
        \toprule
        Sketch dimension & LM-Harness avg. & Training throughput \\
       \midrule 
        32 & 57.57 & 8.37 \\
        \midrule 
        no-sketch  & 58.89 & 7.65 \\
        \bottomrule 
    \end{tabular}
    }
    \label{table:PLS results}
\end{table}

\section{Latent Sketching for Approximate Solutions}
\label{appendix: sketching}
We introduce the idea of sketching from random matrix theory to further control the amount of FLOPs vs accuracy in \ourshortname. 
Sketching involves down projecting the normal equations into a low-dimensional subspace, solving the equations in this subspace and finally up-projecting the solution back to the original space. This reduces the worst-case computational complexity of our approach from $\mathcal{O}(D^2 r)$ to $\mathcal{O}(d^2 r)$, where $d \ll D$ and $r$ is the number of iterations in \cref{algo:Chebyshev}. 
To the best of our knowledge our work is the first one introducing sketching as a viable solution to increase efficiency of neural network layers that are defined implicitly by the solution to an optimization problem. Sketching can be thought of as an analogous to the Multi Latent Attention idea introduced by DeepSeek but applied to fading memory layers. \cref{table:PLS results} shows preliminary results of this idea applied to \ourshortname. Both models (no-sketch and sketch dim 32) are trained from scratch at 2.8B scale on 100B tokens.

\section{Hybrid Gated KalmaNet}
\label{appendix: hybrid}
As discussed in \cref{section:appendix_hybrid_intro}, augmenting SSM models with Attention layers has proven to be an effective way of improving performance on tasks that require recalling information from the distant past. In this section, we show that our Gated KalmaNet layer can be interleaved with Attention layers to yield even stronger models. Our Hybrid GKA model is based on the Qwen3 architecture \cite{yang2025qwen3}. Namely, our Hybrid model consists of a stack of ``decoder'' blocks, each of which contains a sequence mixer---either Attention or GKA---followed by an MLP. Similar to Qwen3, our Attention layers use QK normalization layers. Our Hybrid model consists of 30 decoder blocks, 26 of which use GKA as the sequence mixer, and 4 that use Attention. The Attention decoder blocks are at indices 6, 14, 22, and 29. Our Hybrid models follow the same training procedure as our non-Hybrid models. Specifically, we pretrain our Hybrid model on 100B tokens with a 4K context size, followed by fine-tuning on 25B tokens at a 128K context size. 

When evaluating our pretrained Hybrid model standard NLP benchmarks, we observe that it improves substantially on recall-oriented tasks (FDA \& SWDE) compared to the non-Hybrid model\footnote{Note, our non-hybrid model shares the same architecture as the hybrid with the distinction that all 4 Attention layers are replaced with GKA layers.}, as shown in \cref{table:hybrid_vs_pure_ssm_lm_harness}. Further, when evaluating our fine-tuned long-context model on tasks that require effective modeling of long-range dependencies, we observe a significant improvement across all context lengths, as shown in \cref{fig:hybrid_vs_ssm_helmet}. 

\begin{table*}[htbp]
\centering
\caption{\textbf{Our Hybrid GKA + Attention model improves language modeling performance.} When interleaving Attention layers into our GKA models, we observe a significant improvement on recall-oriented tasks, such as FDA and SWDE, while preserving a similar performance on short-context tasks.}
\label{table:hybrid_vs_pure_ssm_lm_harness}
\small
\resizebox{0.9\textwidth}{!}{
\begin{tabular}{l|cccccccccc|c}
\toprule
\textbf{Model} & \textbf{\begin{tabular}{@{}c@{}}ARC-C\end{tabular}} & \textbf{\begin{tabular}{@{}c@{}}ARC-E\end{tabular}} & \textbf{\begin{tabular}{@{}c@{}}BoolQ\end{tabular}} & \textbf{\begin{tabular}{@{}c@{}}COPA\end{tabular}} & \textbf{\begin{tabular}{@{}c@{}}HellaSWAG\end{tabular}} & \textbf{\begin{tabular}{@{}c@{}}PIQA\end{tabular}} & \textbf{\begin{tabular}{@{}c@{}}SciQ\end{tabular}} & \textbf{\begin{tabular}{@{}c@{}}Winogrande\end{tabular}} & \textbf{\begin{tabular}{@{}c@{}}FDA\end{tabular}} & \textbf{\begin{tabular}{@{}c@{}}SWDE\end{tabular}} & \textbf{\begin{tabular}{@{}c@{}}Avg\end{tabular}} \\
& acc\_n $\uparrow$ & acc\_n  $\uparrow$ & acc $\uparrow$ & acc $\uparrow$ & acc\_n $\uparrow$ & acc\_n $\uparrow$ & acc\_n $\uparrow$ & acc $\uparrow$ & contains $\uparrow$ & contains $\uparrow$ & \\
\midrule
\midrule
Gated KalmaNet (Hybrid) & \textbf{33.02} & 59.47 & \textbf{64.07} & 80.00 & 62.74 & 74.59 & 81.40 & \textbf{64.64} & \textbf{53.18} & \textbf{72.46} & \textbf{64.56} \\
Gated KalmaNet & 32.51 & \textbf{59.89} & 61.68 & \textbf{85.00} & \textbf{63.84} & \textbf{74.81} & \textbf{83.20} & 64.17 & 12.89 & 50.95 & 58.89 \\
\bottomrule
\end{tabular}
}
\end{table*}

\begin{figure}[h]
    \centering
    \includegraphics[width=1.0\textwidth]{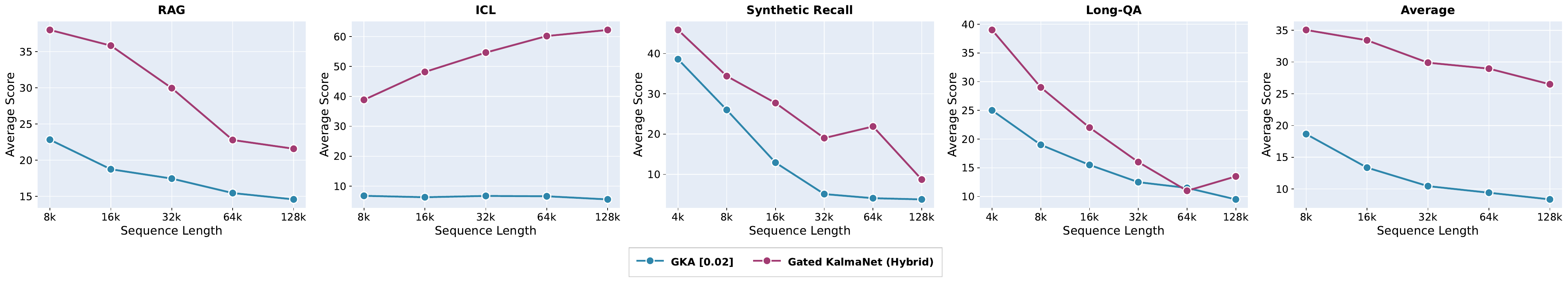}
    \caption{\textbf{Our Hybrid GKA + Attention model significantly improves performance across all long-context benchmarks compared to our non-Hybrid model.} Adding a few Attention layers to our GKA model improves long-range dependency modeling, improving performance across all sequence lengths on RAG, ICL, Synthetic Recall, and Long-QA.}
    \label{fig:hybrid_vs_ssm_helmet}
\end{figure}

\newpage

\section{GKAVision architecture}\label{sec: GKAVision}
We obtain GKAVision by replacing the MambaMixer blocks in the MambaVision architecture \cite{hatamizadeh2025mambavision} with \ourshortname. For the GKA layer, we use the following hyperparameters: 
\begin{enumerate} 
\item We employ convolution kernels of size 3 (same as MambaVision). 
\item We use an expansion factor of 0.5 (\texttt{expand\_k} and \texttt{expand\_v} parameters), which shrinks the input to the layer. This choice was made to ensure comparable model size with MambaVision. 
\item We use a head dimension of 16. This was chosen to maximize the number of heads while ensuring the head dimension does not fall below the minimum requirements of Triton kernels.
\item We added an input-selectivity gate [introduced in \cite{chattopadhyay2026priming}], $\beta$, to the GKA recurrence which we found to be helpful for boosting performance on ImageNet. Specifically, \eqref{eq:ch-forward} is modified to 
\begin{equation}\label{eq:ch-forward-beta}
    \begin{split}
        H_t&= \gamma_t \cdot  H_{t-1} + \beta_tk_t k_t^\top, \ \  U_t = \gamma_t \cdot U_{t-1} + \beta_t v_t k_t^\top, \\
        y_t &= U_t \hat{x}_t, \quad \hat{x}_t = \text{CH}(H_t+\lambda_t I, q_t, r),
    \end{split} 
\end{equation}

where $\beta_t$ is a per-head scalar modelled as a linear projection of the input followed by a sigmoid activation function. In our ablation study, adding this $\beta$ gate improved ImageNet Top-1 accuracy from 81.14 $\to$ 81.27.
\end{enumerate}

Remaining hyperparameters ($\lambda$ and number of CH iterations) are the same as described in \cref{sec: Model architecture}. For training all models on ImageNet, we used the AdamW optimizer with cosine scheduling, a max LR of 5e-4, and weight decay of 0.05. All other training parameters follow the defaults from the official MambaVision \href{https://github.com/NVlabs/MambaVision?tab=readme-ov-file}{repository}.

\end{document}